\documentclass{article}
\pdfoutput=1

\usepackage{times}
\usepackage{graphicx} 
\usepackage{subfigure} 

\usepackage{natbib}

\usepackage{algorithm}
\usepackage{algorithmic}

\usepackage{hyperref}

\usepackage{amsfonts}
\usepackage{amssymb}
\usepackage{amsmath}
\usepackage{epsfig}
\usepackage{multirow}
\usepackage{color}
\usepackage{enumerate}
\usepackage{units}
\usepackage{natbib}
\usepackage{ifthen}
\usepackage{authblk}
\usepackage[margin=1.5in]{geometry}

\def\Pr{{\bf P}}
\def\E{{\mathbb E}}

\newcommand{\alg}[1]{\textsc{#1}}

\DeclareMathOperator{\indicate}{\mathbf{1}} 
\DeclareMathOperator{\classD}{\mathfrak{D}}
\DeclareMathOperator{\R}{\mathbb{R}}
\DeclareMathOperator{\M}{\mathcal{M}}
\DeclareMathOperator{\D}{\mathcal{D}}
\DeclareMathOperator{\Xb}{\mathbf{X}}

\DeclareMathOperator{\distM}{\rho_{_M}}

\DeclareMathOperator{\T}{^\mathsf{T}}    
\DeclareMathOperator{\Hnet}{\mathcal{H}^{\textrm{2-net}}_{\sigma^\gamma}}

\DeclareMathOperator{\fat}{\mathsf{Fat}}
\DeclareMathOperator{\err}{\textup{err}}

\DeclareMathOperator{\mean}{\textup{mean}}
\DeclareMathOperator{\Mgap}{\mathcal{M}_\textup{0-1}}
\DeclareMathOperator{\dist}{\textup{dist}}
\DeclareMathOperator{\reg}{\textup{reg}}
\DeclareMathOperator{\hypoth}{\textup{hypoth}}
\DeclareMathOperator{\cov}{\textrm{cov}}
\DeclareMathOperator{\pak}{\textrm{pak}}
\DeclareMathOperator{\argmin}{\textup{argmin}}

\newcommand{\ie}{{\emph{i.e.,}}}

\newcommand*\xbar[1]{%
  \hbox{%
    \vbox{%
      \hrule height 0.5pt 
      \kern0.17ex
      \hbox{%
        \kern-0.25em
        \ensuremath{#1}%
        \kern-0.12em
      }%
    }%
  }%
}

\def\qed{\vrule height8pt width3pt depth0pt}
\def\jump{\vskip0.05in}

\newtheorem{theorem}{Theorem}                               
\newtheorem{definition}{Definition}                
\newtheorem{corollary}[theorem]{Corollary}                  
\newtheorem{lemma}[theorem]{Lemma}                          
\newtheorem{fact}[theorem]{Fact}                          
\newenvironment{proof}{\noindent{\it Proof.} }{\qed\jump}   

\begin{document} 

\title{Sample Complexity of Learning Mahalanobis \\Distance Metrics}

\author{Nakul Verma\thanks{email: \texttt{verman@janelia.hhmi.org}; corresponding author.} }
\author{Kristin Branson\thanks{email: \texttt{bransonk@janelia.hhmi.org}}}
\affil{Janelia Research Campus\\Howard Hughes Medical Institute, Virginia, USA}
\date{}


\maketitle

\begin{abstract} 
Metric learning seeks a transformation of the feature space that enhances prediction
quality for the given task at hand. In this work we provide PAC-style 
sample complexity rates for supervised metric learning. We give matching lower-
and upper-bounds showing that the sample complexity scales with the representation
dimension when no assumptions are made about the underlying data distribution. 
However, by leveraging the structure of the data distribution, we show that one can achieve rates that are \emph{fine-tuned} to 
a specific notion of intrinsic complexity for a given dataset. Our analysis 
reveals that augmenting the metric learning optimization criterion with a simple norm-based regularization 
can help adapt to a dataset's intrinsic complexity, yielding better
generalization.
Experiments on benchmark datasets validate our
analysis and show that regularizing the metric can help discern the signal even when the data contains high amounts of noise.

%
\end{abstract} 

\allowdisplaybreaks

\section{Introduction}

In many machine learning tasks, data is represented in a
high-dimensional Euclidean space where each dimension corresponds to some
interesting measurement of the observation. Often, 
practitioners include a variety of measurements in hopes that some
combination of these features will capture the relevant information. 
While it is natural to represent such data in a Real space of measurements,
there is no reason to expect that using Euclidean ($L_2$) distances to compare the
observations will be necessarily useful for the task at hand. Indeed, the presence of
uninformative or mutually correlated measurements simply inflates the $L_2$-distances 
between pairs of observations, rendering distance-based comparisons ineffective. 

\emph{Metric learning} has emerged as a powerful technique to learn a good
notion of distance or a \emph{metric} in the representation space that can emphasize the 
feature combinations that help in the predication task while suppressing the contribution of
spurious measurements. The past decade has seen a variety of successful metric learning algorithms 
that leverage various attributes of the problem domain. A few notable examples include exploiting class labels 
to find a Mahalanobis distance metric that maximizes the distance between
dissimilar observations while minimizing distances between similar ones to improve classification quality
\citep{met_learn:alg_LMNN, met_learn:alg_ITML}, and explicitly optimizing for
a downstream prediction task such as information retrieval 
\citep{met_learn:alg_mlr_mcfee}. 



Despite the popularity of metric learning methods, few studies have focused on
studying how the problem complexity scales with key attributes of a given
dataset. For instance, how do we expect the generalization error to scale---both
theoretically and practically---as one varies the number of informative and
uninformative measurements, or changes the noise levels? 

Here we study supervised metric learning more
formally and gain a better understanding of how different modalities in data affect the metric learning problem. 
We develop two general frameworks for PAC-style analysis of supervised metric learning. 
We can categorize the popular metric learning algorithms into an empirical
error minimization problem in one of the two frameworks.  The first generic
framework, the distance-based metric learning framework, uses class label
information to derive distance constraints.
The key objective is to learn a metric that on average yields smaller distances between
examples from the same class than those from
different classes. 
Some popular 
algorithms that optimize for such distance-based objectives include Mahalanobis Metric for Clustering (MMC) by \citet{met_learn:alg_xing} and Information Theoretic Metric Learning 
(ITML) by \citet{met_learn:alg_ITML}.
Instead of using
distance comparisons as a proxy, however, one can also optimize for a specific prediction task directly.
The second generic framework, the classifier-based metric learning framework,
 explicitly incorporates the hypothesis associated with the prediction
task of interest to learn effective distance metrics. 
A few interesting examples in this regime include the work by \citet{met_learn:alg_mlr_mcfee} that finds metrics that improve ranking quality in information retrieval tasks, and
the work by \citet{met_learn:alg_spml_shaw} that learns metrics that help predict connectivity structure in networked data.

Our analysis shows that in both frameworks, the sample complexity
scales with the representation dimension for a given dataset (Lemmas \ref{lm:unif_conv_all} and \ref{lm:hypoth_ub}), and
this dependence is necessary in the absence of any specific assumptions on the underlying
data distribution (Lemmas \ref{lm:lb_dist} and \ref{lm:hypoth_lb}). By
considering any Lipschitz loss, our results generalize previous sample
complexity results (see our discussion in Section
\ref{sec:related_work}) and, for the first time in the literature, provide
matching lower bounds.

In light of the observation made earlier that data measurements often include uninformative or weakly informative features, 
we expect a metric that yields good generalization performance to de-emphasize such features and accentuate the relevant ones.
We can thus formalize the \emph{metric learning complexity} of a given dataset in terms of
the intrinsic complexity $d$ of the metric that reweights the features in a way that yields the best generalization
performance. (For Mahalanobis distance metrics, we can characterize the intrinsic complexity by the \emph{norm} of the
matrix representation of the metric.) We refine our
sample complexity result and show a \emph{dataset-dependent} bound for both frameworks that scales with 
dataset's intrinsic metric learning complexity $d$ (Corollary \ref{cor:unif_conv_refined}). 

Taking guidance from our dataset-dependent result, we propose a simple variation on
the empirical risk minimizing (ERM) algorithm that, when given an i.i.d.\ sample, 
returns a metric (of complexity $\hat{d}$) that jointly minimizes the observed
sample bias and the expected intra-class variance for metrics of fixed
complexity $\hat{d}$. This bias-variance balancing algorithm can be viewed as a
structural risk minimizing algorithm that provides better generalization
performance than an ERM algorithm and justifies norm-regularization of
weighting metrics in the optimization criteria for metric learning. 

Finally, we evaluate the practical efficacy of our proposed norm-regularization criteria with some popular metric
learning algorithms on benchmark datasets (Section \ref{sec:experiments}). Our experiments highlight that the norm-regularization indeed helps
in learning weighting metrics that better adapt to the signal in data in high-noise regimes.

\section{Preliminaries}

Given a representation space $X = \R^D$ of $D$ real-valued measurements of 
observations of interest, the goal of metric learning is to learn a
\emph{metric} $M$ (that is, a $D\times D$ real-valued weighting matrix on
$X$; to remove arbitrary scaling we shall assume that the
maximum singular value of $M$, that is, $\sigma_{\max}(M) = 1$)\footnote{Note
 that we are 
looking at the linear form of the metric $M$; usually the corresponding quadratic form $M^\mathsf{T}M$ is
discussed in the literature, which is necessarily positive semi-definite.}
that minimizes some notion of \emph{error} on data drawn from an unknown underlying
distribution $\D$ on $X \times \{0,1\}$. Specifically, we want to find the metric 
\begin{eqnarray*}
M^* := \argmin_{M \in \mathcal{M}} \err(M,\D),
\end{eqnarray*}
from the class of metrics $\mathcal{M}$ under consideration, that is, $\mathcal{M}:= \{ M \;|\; M\in \R^{D\times D}, \sigma_{\max}(M)=1  \}$.
For supervised metric learning, this \emph{error} is typically label-based and can be defined in multiple
reasonable ways. As discussed earlier, we explore two intuitive regimes for defining error.
\\

\textbf{Distance-based error.} 
A popular criterion for quantifying error in metric learning is by comparing \emph{distances} amongst 
points
drawn from the underlying data distribution.  
Ideally, we want a weighting metric $M$ that brings data from the same class closer
together than those from opposite classes. In a distance-based framework, a natural way to accomplish this is to
find a weighting $M$ that yields shorter distances between pairs of observations from the same class
than those from different classes.
By penalizing
how often and by how much the distances violate these constraints gives rise
to the particular form of the error.

Let the variable $z=(x,y)$ denote a random draw from $\D$ with $x \in X$ as 
the observation and $y\in\{0,1\}$ its associated label,
and let $\lambda$ denote 
how severely one wants to penalize the distance violations, 
then a natural definition of \emph{distance-based} error becomes: 
\begin{align*}
\err_{\dist}^\lambda(M,\D) := \E_{z_1, z_2 \sim \D} \Big[ \phi^\lambda\big( \distM(x_1,x_2), Y \big) \Big],  
\end{align*}
for a generic distance-based loss function $\phi^\lambda(\distM, Y)$, that computes
the degree of violation between weighted distance $\distM(x_1,x_2) := \|M(x_1-x_2)\|^2$ and the label agreement $Y := \indicate[y_1 = y_2]$
among a pair $z_1 = (x_1,y_1)$ and $z_2 = (x_2,y_2)$ drawn from 
$\D$.

%
%
%
%
An example instantiation of $\phi$ popular in literature encourages metrics
that yield distances that are no more than some upper limit $U$ between
observations from the same class, and distances that are no less than some
lower limit $L$ between those from
different classes (for some $U<L$). Thus
\begin{align}
\phi_{L,U}^\lambda(\distM,Y) := \Bigg\{ 
\begin{array}{ll} 
  \min\{1, \lambda[\distM - U]_{_+} \} & \textrm{if $Y = 1$} \\ 
  \min\{1, \lambda[L - \distM]_{_+} \}  & \textrm{otherwise} \end{array},
\label{eq:pair_LUloss}
\end{align}
where $[A]_{_+} := \max\{0,A\}$.

 \citet{met_learn:alg_xing} optimize an efficiently computable variant of
this criterion, in which they look for a metric that keeps the total pairwise distance
amongst the observations from the same class less than a constant while
maximizing the total pairwise distance amongst the observations from 
opposite classes. The variant proposed by \citet{met_learn:alg_ITML} explicitly includes the
upper and lower limits with an added regularization on the learned $M$
to be close to a pre-specified metric of interest $M_0$. 

While we discuss loss-functions $\phi$ that handle distances between a \emph{pair} of observations, it is
easy to extend to distances among \emph{triplets}. Rather than 
having hard upper and lower limits which every pair of the same and the opposite
classes must obey, a triplet-based comparison typically focuses on relative distances between three
observations at a time. A natural instantiation in this case becomes:
\begin{align*}
&\phi_{\textrm{triple}}^\lambda(\distM(x_1,x_2), \distM(x_1,x_3), (y_1,y_2,y_3)) := 
\Bigg\{ \!\! \begin{array}{ll} 
  \min\{1,\lambda [\distM(x_1,x_2) - \distM(x_1,x_3)]_{_+}\}  & \!\!\! \textrm{if $y_1 = y_2 \neq y_3$} \\ 
0 & \!\!\! \textrm{otherwise} 
\end{array} \!\!, 
\end{align*}
for a triplet $(x_1,y_1)$, $(x_2,y_2)$, $(x_3,y_3)$
drawn from $\D$.

 \citet{met_learn:alg_LMNN} discuss an interesting variant of this,
in which instead of looking at all triplets in a given training sample, they focus on
triplets of observations in local neighborhoods and learn a metric that maintains
a gap or a margin among distances between observations from the same class and those from the
opposite class. Improving the quality of distance comparisons in local
neighborhoods directly affects the nearest neighbor performance, making this a
popular technique. 
\\

\textbf{Classifier-based Error.} Distance comparisons
typically act as a surrogate for a specific downstream prediction task.
If we want a metric that directly optimizes for a task,
we need to explicitly incorporate the hypothesis class being used for that
task while finding a good weighting metric. 

This simple but effective insight has been used recently by
\citet{met_learn:alg_mlr_mcfee} for improving ranking results in information retrieval problems by explicitly
incorporating ranking losses while learning an effective weighting metric.
\citet{met_learn:alg_spml_shaw} also follow this principle and explicitly include network topology constraints
to learn a weighting metric that can better predict the connectivity structure in
social and web networks.

We can formalize the classifier-based metric learning framework by considering a fixed hypothesis class $\mathcal{H}$ of interest on the measurement
domain. To keep the discussion general, we shall assume that the hypotheses are real-valued and can be regarded as a measure of confidence 
in classification, that is, each $h\in \mathcal{H}$ is of the form $h:
X\rightarrow [0,1]$. (One can obtain the binary predictions from $h$ by a simple thesholding at $1/2$.) Then, the error induced by a particular weighting metric $M$ on the measurement space $X$ can be defined as the best possible error that can be obtained by hypotheses in $\mathcal{H}$, that is
\begin{eqnarray*}
\err_{\hypoth}(M,\D) := \inf_{h\in\mathcal{H}} \E_{(x,y) \sim \D} \Big[ \indicate \big[ | h(Mx) - y| \geq 1/2 \big] \Big].
\end{eqnarray*}

We shall study how this error scales with various key parameters of the metric learning problem.

\section{Learning a Metric from Samples}

In any practical setting, we estimate the ideal weighting metric $M^*$ by minimizing the
empirical version of the error criterion from a finite size sample from $\D$.  

Let $S_m$ denote a sample of size $m$, and $\err(M,S_m)$ denote the empirical error on the sample $S_m$ (the exact definitions of $S_m$ and the form
of $\err(M,S_m)$ are discussed later). We can then define the empirical risk minimizing metric based on $m$ samples as
$M^*_m := \argmin_M
\err(M,S_m).$ Most practical algorithms, of course, return some approximation of $M^*_m$, and thus it is important to compare the generalization ability of
$M^*_m$ to that of theoretically optimal $M^*$. That is, how
\begin{equation}
 \err(M^*_m, \D) - \err(M^*,\D)   
\label{eq:erm_conv}
\end{equation}
converges as the sample size $m$ grows.

\subsection{Distance-Based Error Analysis} 
Given an i.i.d.\ sequence of observations $z_1,z_2,\ldots$ from $\D$, we can pair the observations
together to form a \emph{paired} sample $S_m = \{(z_1,z_2),(z_3,z_4),\ldots,(z_{2m-1},z_{2m}) \} = \{(z_{1,i},z_{2,i})\}_{i=1}^m$ of size $m$, and define
the sample based distance error $\err_{\dist}^\lambda(M, S_m)$ induced by a metric $M$ as
\begin{eqnarray*}
\err_{\dist}^\lambda(M, S_m) := \frac{1}{m} \sum_{i=1}^m \phi^\lambda\big( \distM(x_{1,i}, x_{2,i}), \indicate[ y_{1,i} =  y_{2,i} ] \big).
\end{eqnarray*}

Then for any bounded support distribution $\D$ (that is, each $(x,y)\sim \D$, $\|x\|\leq B <\infty$), we have the following convergence result.\footnote{We only 
present the results for paired distance comparisons; the results are easily extended to triplet-based comparisons.}

\begin{lemma}
\label{lm:unif_conv_all}
Fix any sample size $m$, and let $S_m$ be an i.i.d.\ \emph{paired} sample of
size $m$ from an unknown bounded distribution $\D$ (with bound $B$). For any distance-based loss function $\phi^\lambda$ that is $\lambda$-Lipschitz in the first argument, with probability at least $1-\delta$ over the draw of $S_m$,
\begin{align*}
 \sup_{M\in \mathcal{M}}  \big[ \err_{\dist}^\lambda  (M,\D) - &
\err_{\dist}^\lambda (M,S_m) \big]
\leq O\Bigg( \lambda B^2 \sqrt{\frac{D \ln(1/\delta)}{m}}\Bigg).
\end{align*}
\end{lemma}

Using this lemma we can get the desired convergence rate (Eq.\ \ref{eq:erm_conv}). Fix $M^*\in \mathcal{M}$, then for any $0<\delta<1$ and $m\geq1$, 
with probability at least $1-\delta$,
we have
\begin{align*}
 \err_{\dist}^\lambda(M^*_m, \D)  - \err_{\dist}^\lambda &(M^*,\D)  
 \\
= &\;\;
    \err_{\dist}^\lambda(M^*_m, \D) - \err_{\dist}^\lambda(M^*_m,S_m)  
 +   \err_{\dist}^\lambda(M^*_m, S_m) - \err_{\dist}^\lambda(M^*,S_m) \\
&\;\; +   \err_{\dist}^\lambda(M^*, S_m) - \err_{\dist}^\lambda(M^*,\D)  \\
%
%
\leq&\;\; O\Bigg( \lambda B^2 \sqrt{\frac{D \ln(1/\delta)}{m}}\Bigg) + \sqrt{\frac{\ln(2/\delta)}{2m}}  \\
=&\;\; O\Bigg( \lambda B^2 \sqrt{\frac{D \ln(1/\delta)}{m}}\Bigg),
\end{align*}
by noting (i) $\err_{\dist}^\lambda(M^*_m, S_m)  \leq \err_{\dist}^\lambda(M^*,S_m)$, since $M^*_m$ is empirical error minimizing on $S_m$, 
and (ii) by using Hoeffding's inequality on the fixed $M^*$ to conclude
that with probability at least $1-\delta/2$, 
$\err_{\dist}^\lambda(M^*, S_m) - \err_{\dist}^\lambda(M^*,\D) \leq \sqrt{\frac{\ln(2/\delta)}{2m}}$.

Thus to achieve a specific estimation error rate $\epsilon$, the number of
samples $m = \Omega\Big(\big(\frac{\lambda B^2}{\epsilon}\big)^2 D \ln(\frac{1}{\delta})\Big)$ are sufficient to
conclude, with confidence at least $1-\delta$, the empirical risk minimizing metric
$M^*_m$ will have estimation error of at most $\epsilon$.
This shows that one never needs more than a number proportional to the representation dimension $D$ examples to achieve the desired level of accuracy. 

Since
typical applications have a large representation dimension, it is instructive to study if such a strong dependency on $D$ necessary. 
It turns out that even for simple distance-based loss functions like $\phi^\lambda_{L,U}$ (cf.\ Eq.\ \ref{eq:pair_LUloss}), 
there are data distributions for which one cannot get away with fewer than linear in $D$ samples and ensure good estimation errors.
In particular we have the following.
%
%
%
\begin{lemma}
\label{lm:lb_dist}
Let $\mathcal{A}$ be any algorithm that, given an i.i.d.\ sample $S_m$ (of size $m$) from a
fixed unknown bounded support distribution $\D$, returns a weighting metric 
from $\mathcal{M}$ that minimizes the empirical error with respect to distance-based loss function $\phi^\lambda_{L,U}$. 
There exist $\lambda\geq 0$, $0\leq U<L$, such that for
all $0<\epsilon,\delta<1/64$, there  
exists a bounded
support distribution $\D$, such that if $m \leq \frac{D+1}{512\epsilon^2}$ then
$$
\Pr_{S_m} \Big[ \err_{\dist}^\lambda(\mathcal{A}(S_m),\D) - \err_{\dist}^\lambda(M^*,\D) > \epsilon \Big] > \delta.
$$
\end{lemma}

While this may seem discouraging for large-scale applications of metric learning, note that here we made no assumptions about the underlying structure of
the data distribution $\mathcal{D}$, making this a worst-case analysis. As the individual features in real-world datasets contain
varying amounts of information for good classification performance, one hopes for a more relaxed dependence on $D$ for metric learning in these settings. This is explored in Section 
\ref{sec:discrim_complex}.

\subsection{Classifier-Based Error Analysis}


In this setting, we can use an i.i.d.\ sequence of observations
$z_1,z_2,\ldots$ from $\D$ to obtain the sample $S_m = \{z_i\}_{i=1}^m$ of size
$m$ directly. To analyze the generalization ability of the weighting
metrics optimized with respect to an underlying hypothesis class $\mathcal{H}$, we need to 
effectively analyze the classification complexity of $\mathcal{H}$. The
scale sensitive version of VC-dimension, also known as the ``fat-shattering dimension", of a real-valued hypothesis class (denoted by
$\fat_\gamma(\mathcal{H})$) encodes the right notion of classification
complexity and provides an intuitive way to relate the generalization error to the
empirical error at a \emph{margin} $\gamma$ (see for instance the work of
\citet{lt:anthony_bartlett} for an excellent discussion).

In the context of metric learning with respect to a fixed hypothesis class, define the empirical error at a margin $\gamma$ as
\begin{align*}
\err_{\hypoth}^\gamma&(M,S_m) := 
 \inf_{h \in\mathcal{H}}\frac{1}{m} \sum_{(x_i,y_i)\in S_m} \indicate[\mathsf{Margin}(h(Mx_i),y_i)<\gamma] ,
\end{align*}
where $\mathsf{Margin}(\hat{y},y) := \Big\{ \begin{array}{ll} \hat{y}-1/2 & \textrm{if $y=1$} \\ 1/2-\hat y & \textrm{otherwise} \end{array}.$


Then for any bounded support distribution $\D$ (that is, each $(x,y)\sim \D$,
$\|x\|\leq B <\infty$), we have the following convergence result that relates
the estimation error rate of the weighting metrics with that of the
fat-shattering dimension of the underlying base hypothesis class.

\begin{lemma}
\label{lm:hypoth_ub}
Let $\mathcal{H}$ be a $\lambda$-Lipschitz base hypothesis class.
Pick any $0<\gamma<1/2$, and let $m\geq \fat_{\gamma/16}(\mathcal{H})\geq 1$.
Then with probability at least $1-\delta$ over an i.i.d.\ draw of sample $S_m$
(of size $m$) from a bounded unknown distribution $\D$ (with bound $B$) on
$X\times \{0,1\}$, 
\begin{align*}
\sup_{M\in \mathcal{M}} \Big[  \err_{\hypoth}(M,\D) - \err^\gamma_{\hypoth}(M,S_m) \Big]  
\leq O\Bigg( 
\sqrt{\frac{1}{m}\ln\frac{1}{\delta} + \frac{D^2}{m}\ln\frac{D}{\epsilon_0} + \frac{\fat_{\gamma/16}(\mathcal{H})}{m}\ln\Big(\frac{m}{\gamma}\Big) }
 \Bigg).
\end{align*}
where $\epsilon_0 := \min\{\frac{\gamma}{2},\frac{1}{2\lambda B} \}$, and $\fat_{\gamma/16}(\mathcal{H})$ is the \emph{fat-shattering dimension} of the
base hypothesis class $\mathcal{H}$ at margin $\gamma/16$.
\end{lemma}

Using a similar line of argument as before, we can bound the key quantity of interest (Eq.\ \ref{eq:erm_conv}) and conclude
%
for any $0<\gamma<1/2$ and any $m\geq 1$, with probability $\geq 1-\delta$
\begin{align*}
 \err_{\hypoth}&(M^*_m, \D) - \err_{\hypoth}^\gamma(M^*,\D)  
= O\Bigg( 
\sqrt{\frac{D^2\ln(D/\epsilon_0)}{m} + \frac{\fat_{\gamma/16}(\mathcal{H}) \ln(m/ \delta \gamma)}{m} }
 \Bigg).
\end{align*}
Here $\epsilon_0 = \min\{\frac{\gamma}{2},\frac{1}{2\lambda B} \}$ for a $\lambda$-Lipschitz hypothesis class $\mathcal{H}$.
Thus to achieve a specific estimation error rate $\epsilon$, the number of
samples $m = \Omega\Big(\frac{D^2\ln(\lambda DB/\gamma)+\fat_{\gamma/16}(\mathcal{H})\ln(1/\delta\gamma)}{\epsilon^2}\Big) $ suffices to
say, with confidence at least $1-\delta$, the empirical risk minimizing metric
$M^*_m$ will have estimation error at most $\epsilon$.

It is interesting to note that the task of finding an optimal metric
only additively increases the sample complexity over the complexity of finding
the optimal hypothesis from the underlying hypothesis class. 



In contrast to the sample complexity of distance-based framework (c.f.\ Lemma \ref{lm:unif_conv_all}), here we get a quadratic dependence 
on the representation dimension. 
The following lemma shows that a strong dependence on the representation dimension is necessary in absence of any specific assumptions on 
the underlying data distribution and the base hypothesis class.

\begin{lemma}
\label{lm:hypoth_lb}
Pick any $0<\gamma<1/8$.
Let $\mathcal{H}$ be a base hypothesis class of $\lambda$-Lipschitz
functions mapping from $X=\R^D$ into the
interval $[1/2-4\gamma,1/2+4\gamma]$ that is closed under addition of
constants. That is
$$h \in \mathcal{H} \implies h' \in \mathcal{H}, \textrm{ where }  h':x\rightarrow h(x)+c \;\;\; \textrm{ for all $c$.} $$
Then for any 
classification algorithm $\mathcal{A}$, and for any $B\geq 1$, there exists $\lambda\geq 0$, for all $0<\epsilon,\delta<1/64$, there exists a bounded support distribution $\D$ (with bound $B$) such that if $m \ln^2 m < O\big(\frac{D^2 + d}{\epsilon^2 \ln(1/\gamma^2) }\big)$ 
$$
\Pr_{S_m\sim\D}[\err_{\hypoth}(h^*,\D) > \err^\gamma_{\hypoth}(\mathcal{A}(S_m),\D) + \epsilon] > \delta,
$$
where $d := \fat_{768\gamma}(\mathcal{H}) $ is the \emph{fat-shattering dimension} of $\mathcal{H}$ at margin $768\gamma$.
\end{lemma}

\section{Data with Uninformative and Weakly Informative Features}
\label{sec:discrim_complex}
Different measurements have varying degrees of ``information content" for the
particular supervised classification task of interest.  Any algorithm or 
analysis that studies the design of effective comparisons between observations
must account for this variability.

To get a solid footing for our study, we introduce the concept of
\emph{metric learning complexity} of a given dataset.
Our key observation is that a metric that yields good generalization performance
should emphasize relevant features while suppressing the contribution of
spurious features. Thus, a good metric reflects the quality of individual feature measurements of data and their relative value for the learning task.
We can leverage this and define the metric learning complexity of
a given dataset as the \emph{intrinsic complexity} $d$ of the weighting metric that
yields the best generalization performance for that dataset (if multiple
metrics yield best performance, we select the one with minimum $d$).
A natural way to characterize the intrinsic complexity of a weighting metric $M$
is via the norm of the matrix representation of $M$. 
Using metric learning complexity as our gauge for the richness of the feature set in a 
given dataset, we can refine our analysis in both our canonical metric learning frameworks.


\subsection{Distance-Based Refinement}
We start with the following refinement of the distance-based metric learning
sample complexity for a class of Frobenius norm-bounded weighting metrics.


\begin{lemma}
\label{lm:unif_conv_dist_refi}
Let $\mathcal{M}$ be any class of weighting metrics on the feature space $X = \R^D$.
Fix any sample size $m$, and let $S_m$ be an i.i.d.\ \emph{paired} sample of
size $m$ from an unknown bounded distribution $\D$ on $X\times \{0,1\}$ (with bound $B$). For any distance-based loss function $\phi^\lambda$ that is $\lambda$-Lipschitz in the first argument, with probability at least $1-\delta$ over the draw of $S_m$,
\begin{align*}
 \sup_{M\in \mathcal{M}}  \big[ \err_{\dist}^\lambda  (M,\D) - &
\err_{\dist}^\lambda (M,S_m) \big]
\leq O\Bigg( \lambda B^2  \sqrt{\frac{d\ln(1/\delta)}{m}}\Bigg),
\end{align*}
where $d$ is a uniform upperbound on the Frobenius norm of the quadratic form of weighting metrics in $\mathcal{M}$, that is, $\sup_{M\in\mathcal{M}} \|M^\mathsf{T} M\|^2_{_F }\leq d$.
\end{lemma}
Observe that if our dataset has a low metric learning complexity (say, $d \ll D$), then
considering an appropriate class of norm-bounded weighting metrics can help sharpen the sample complexity
result, yielding a \emph{dataset-dependent} bound. We discuss how to automatically adapt
to the right complexity class in Section \ref{sec:auto_refine} below.

\subsection{Classifier-Based Refinement}

Effective data-dependent analysis of classifier-based metric learning requires accounting for 
potentially complex interactions between an arbitrary base hypothesis class and the
distortion induced by a weighting metric to the unknown underlying data
distribution. 
%
To make the analysis tractable while still keeping
our base hypothesis class $\mathcal{H}$ general, we shall assume that
$\mathcal{H}$ is a class of two layer feed-forward neural networks. Recall that for 
any smooth target function $f^*$, a two layer feed-forward neural network (with appropriate 
number of hidden units and connection weights) can approximate $f^*$ arbitrarily well \citep{nnet:univ_approx_bddfxn}, so
this class is flexible enough to incorporate most reasonable target hypotheses.

More formally, define the base hypothesis class of two layer feed-forward neural network with $K$ hidden units as
\begin{equation*}
\Hnet := \Big\{ x \mapsto \sum_{i=1}^K w_i \; \sigma^\gamma(v_i \; \cdot \; x) \; \Big| \; \|w\|_1 \leq 1, \|v_i\|_1 \leq 1  \Big\},
\end{equation*}
where $\sigma^\gamma: \R \rightarrow [-1,1]$ is a smooth, strictly monotonic, $\gamma$-Lipschitz activation function with $\sigma^\gamma(0)=0$. 
Then for the generalization error of a weighting metric $M$ defined with respect to any classifier-based $\lambda$-Lipschitz loss function $\phi^\lambda$
\begin{equation*}
\err^\lambda_{\hypoth}(M,D):= \inf_{h\in\mathcal{\Hnet}} \E_{(x,y)\sim \D} \big[ \phi^\lambda \big(h(Mx),y\big) \big],
\end{equation*}
we have the following.\footnote{Since we know the functional form of the base hypothesis class $\mathcal{H}$ (\ie~a two layer feed-forward neural net), we can provide a more precise bound than leaving it as $\fat(\mathcal{H})$.}

\begin{lemma}
\label{lm:unif_conv_clf_refi}
Let $\mathcal{M}$ be any class of weighting metrics on the feature space $X = \R^D$. For any $\gamma>0$, 
let $\Hnet$ be a two layer feed-forward neural network base hypothesis class (as defined above) and $\phi^\lambda$ be a classifier-based loss function that $\lambda$-Lipschitz in its first argument. Fix any sample size $m$, and let $S_m$ be an i.i.d.\ sample of
size $m$ from an unknown bounded distribution $\D$ on $X\times \{0,1\}$ (with bound $B$). Then with probability at least $1-\delta$,
\begin{align*}
 \sup_{M\in \mathcal{M}}  \big[ \err_{\hypoth}^\lambda (M,\D) - 
&\err_{\hypoth}^\lambda(M,S_m)  \big]  
\leq O\Bigg(B\lambda \gamma \sqrt{\frac{d\ln(D/\delta)}{m} } \Bigg),
\end{align*}
where $d$ is a uniform upperbound on the Frobenius norm of the quadratic form of weighting metrics in $\mathcal{M}$, that is, $\sup_{M\in\mathcal{M}} \|M^\mathsf{T} M\|^2_{_F }\leq d$.
\end{lemma}

\subsection{Automatically Adapting to Intrinsic Complexity}
\label{sec:auto_refine}

Note that while Lemmas \ref{lm:unif_conv_dist_refi} and \ref{lm:unif_conv_clf_refi} provide a
sample complexity bound that is tuned to the metric learning complexity of a given dataset, these results are \emph{not} useful directly since one cannot select the correct
norm bounded class $\mathcal{M}$ a priori (as the underlying
distribution $\D$ is unknown). 
 
Fortunately, by considering an appropriate sequence of norm-bounded 
classes of weighting metrics, we can provide a uniform bound that \emph{automatically adapts} to the intrinsic complexity of the unknown underlying data distribution $\D$.
In particular, we have the following.
\begin{corollary} 
\label{cor:unif_conv_refined}
Fix any $m$, and let $S_m$ be an i.i.d.\ sample of size $m$ from an unknown bounded distribution $\D$ (with bound $B$). 
Define $\mathcal{M}^d := \{M \;|\; \|M^\mathsf{T}M\|^2_{_F} \leq d \}$, and consider
the nested sequence of weighting metric class $\mathcal{M}^1 \subset
\mathcal{M}^2 \subset \cdots$. Let $\mu_d$ be any non-negative measure across
the sequence $\mathcal{M}^d$ such that $\sum_d \mu_d = 1$ (for $d=1,2,\cdots$).
Then for any $\lambda\geq 0$, with probability at least $1-\delta$, for all $d=1,2,\cdots$, and all $M^d \in \mathcal{M}^d$,
\begin{align}
 \big[ \err^\lambda  (M^d,\D) - 
 \err^\lambda (M^d,S_m)  \big]  
\leq O\Bigg(C \cdot B\lambda \sqrt{\frac{d\ln(1/\delta \mu_d)}{m}} \Bigg),
\label{eq:main_unif_refind_bound}
\end{align}
where $C:=B$ for distance-based error, or $C:=\gamma\sqrt{\ln{D}}$ for classifier-based error (with base hypothesis class $\Hnet$).


In particular, for a data distribution $\D$ that has metric learning complexity
at most $d \in \mathbb{N}$, if there are $m \geq
\Omega\Big(\frac{d(CB\lambda)^2  \ln(1/\delta\mu_d)}{\epsilon^2} \Big)$ samples, then with probability at least
$1-\delta$
\begin{align*}
\nonumber \big[ \err^\lambda  (M_m^{\reg},\D) - \err^\lambda (M^*,\D)  \big] \; \leq \; O( \epsilon), 
\end{align*}
for ${ M_m^{\reg}\!:=\! \argmin_{M\in \mathcal{M}}  \Big[ \err^\lambda(M,S_m) + \Lambda_{_M} d_{_M} \Big]}$, 
where 
$\Lambda_{_M} := C B \lambda \sqrt{{\ln\Big(\frac{1}{\delta\mu_{_{d_{M}}}}\Big)}/{m}}\;$
and 
$d_{_M}:= \big\lceil\; \|M^\mathsf{T}M\|^2_{_F} \big\rceil $ 
. 
\end{corollary}

Observe that the measure $(\mu_d)$ above encodes our prior belief 
on the complexity class $\mathcal{M}^d$ from which a target metric is selected by a metric learning algorithm given the training sample $S_m$. In absence of any prior beliefs, $\mu_d$ can be simply set to $1/D$ (for $d=1,\ldots,D$) for unit spectral-norm weighting metrics.

Thus, for an unknown underlying data distribution $\D$ with metric learning
complexity $d$, with number of samples just proportional to $d$, we can find a good weighting metric. 

This result also highlights that the generalization error of \emph{any} weighting metric returned by an
algorithm is proportional to the (smallest) norm-bounded class to which it
belongs (cf.\ Eq.\ \ref{eq:main_unif_refind_bound}). If two metrics $M_1$ and
$M_2$ have similar empirical errors on a given sample, but have different 
intrinsic complexities, then the expected risk of the two metrics 
can be considerably different. We expect the metric with lower intrinsic complexity to yield better
generalization error.
This partly explains the observed empirical success of various types of norm-regularized optimization criteria for
finding the optimal weighting metric \citep{met_learn:alg_rob_stru_mcfee, met_learn:alg_fantope_regularization}.

Using this as a guiding principle, we can design an improved optimization
criteria for metric learning problems that jointly minimizes the sample error and
a Frobenius norm regularization penalty. In particular, 
\begin{align}
\min_{M\in\mathcal{M}} \hspace{0.2in} \err(M,S_m) \hspace{0.1in} + \hspace{0.1in} \Lambda  \; \|M^\mathsf{T}M\|^2_{_F}
\label{eq:opt_regularize}
\end{align}
for any error criteria `$\err$' used in a downstream prediction task of interest and a regularization hyper-parameter $\Lambda$ proportional to $m^{-1/2}$. 
We explore the practical efficacy of this augmented optimization on some representative applications below.

%


%
%
%
%
%
%
%

\begin{figure*}[t]
\begin{center}
\includegraphics[width=1.81in]{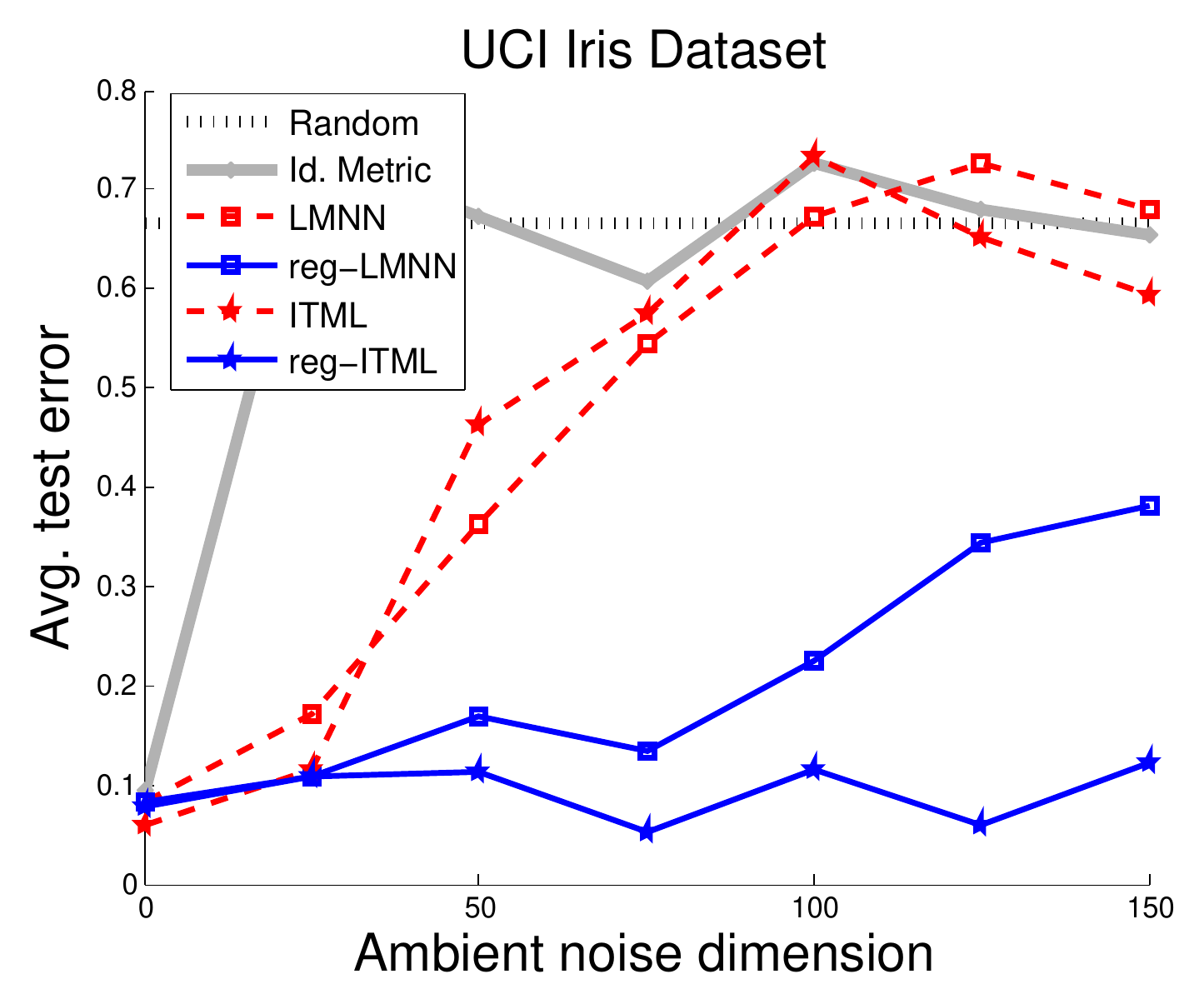}
\includegraphics[width=1.81in]{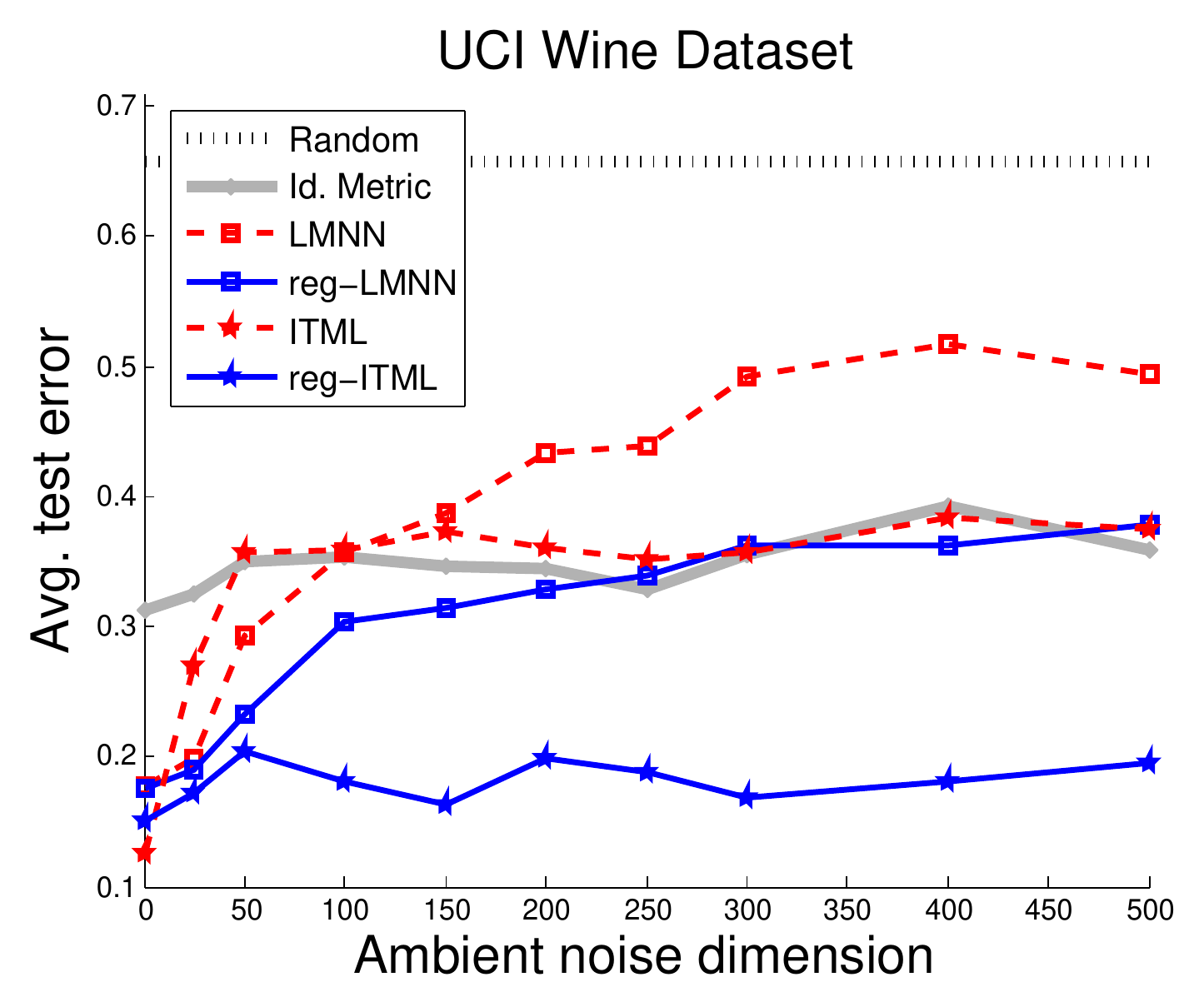}
\includegraphics[width=1.81in]{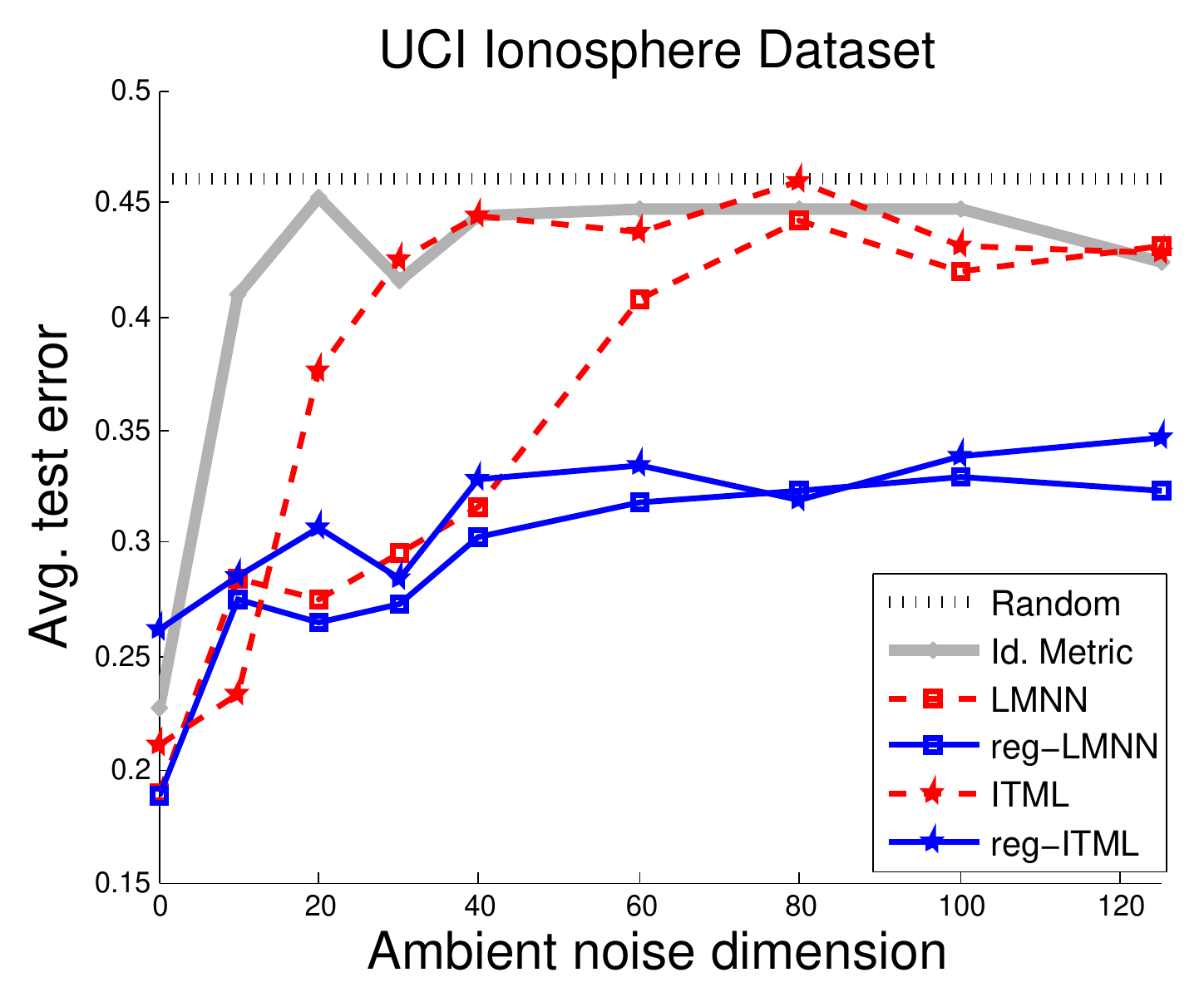}
\caption{\small Nearest-neighbor classification performance of LMNN and ITML metric learning algorithms without regularization (dashed red lines) and with regularization (solid blue lines) on benchmark UCI datasets.
The horizontal dotted line is the classification error of random label assignment drawn according to the class proportions, and solid gray line shows classification error of $k$-NN performance with respect to identity metric (no metric learning) for baseline reference.
}
\label{fig:UCI_NN}
\end{center}
\vskip -0.2in
\end{figure*}

\section{Empirical Evaluation}
\label{sec:experiments}

Our analysis shows that the generalization error of metric learning 
can scale with the representation dimension, and regularization can help mitigate this 
by adapting to the intrinsic \emph{metric learning complexity} of the given
dataset. We want to explore to what degree these effects manifest in practice. 

We select two popular metric learning algorithms, LMNN by \citet{met_learn:alg_LMNN} and ITML by \citet{met_learn:alg_ITML}, that are designed to find
metrics that improve nearest-neighbor classification quality. 
These algorithms have varying degrees of regularization built into their optimization criteria:
LMNN implicitly regularizes the metric
via its ``large margin" criterion, while ITML allows for explicit regularization by letting the practitioners specify a ``prior" weighting metric.  
We modified the LMNN optimization criteria as per Eq.\ \eqref{eq:opt_regularize} to also allow for an explicit norm-regularization controlled by the trade-off parameter $\Lambda$.

We can evaluate how the unregularized criteria (\ie~unmodified LMNN, or ITML with the prior set to the identity matrix) compares to the regularized criteria (\ie~modified LMNN with best $\Lambda$, or ITML with the prior set to a low-rank matrix).
\\


\noindent \textbf{Datasets.} 
We use the UCI benchmark datasets for our experiments:
\alg{Iris} (4 dim., 150 samples), \alg{Wine} (13 dim., 178 samples) and \alg{Ionosphere} (34
dim., 351 samples) datasets \citep{dataset:UCI_ML_repo}. Each dataset has a
fixed (unknown) intrinsic dimension; we can vary the representation dimension by 
augmenting each dataset with synthetic correlated noise of varying dimensions,
simulating regimes where datasets contain
large numbers of uninformative features. 

Each UCI dataset is augmented with synthetic $D$-dimensional correlated noise as follows. We first sample a covariance matrix $\Sigma_D$ from
unit-scale Wishart distribution (that is, let $A$ be a $D\times D$ Gaussian random matrix with
entry $A_{ij} \sim N(0,1)$ drawn i.i.d., and set $\Sigma_D := A^\mathsf{T}A$). Then each sample $x_i$ from the dataset is appended independently by drawing
noise vector $x_\sigma \sim N(0,\Sigma_D)$.
\\

%
%

\noindent \textbf{Experimental setup.}
We varied the ambient noise dimension $D$ between 0 and 500 dimensions and added it to the UCI datasets, creating the noise-augmented datasets.
Each noise-augmented dataset was randomly split between 70\% training, 10\%
validation, and 20\% test samples. 

We used the default settings for each algorithm. 
For regularized LMNN, we picked the best performing trade-off parameter $\Lambda$ from $\{0,0.1,0.2,...,1\}$ on the validation set.
For regularized ITML, we seeded with the rank-one
discriminating metric, \ie~we set the prior as the matrix with all zeros,
except the diagonal entry corresponding to the most
discriminating coordinate set to one.

All the reported results were averaged over 20 runs. 
\\

%
%


\noindent \textbf{Results.}
%
Figure \ref{fig:UCI_NN} shows the nearest-neighbor performance (with $k=3$) of LMNN and ITML on noise-augmented UCI datasets.
Notice that the unregularized versions of both algorithms (dashed red lines) scale poorly when noisy features are introduced. 
As the number of uninformative features grows, the performance of both
algorithms quickly degrades to that of classification performance in the original unweighted space with no metric learning (solid gray line),
showing poor adaptability to the signal in the data.

Interestingly, neither of the unregularized algorithms performs consistently 
better than the other on datasets with high noise:
ITML yields better results on \alg{Wine}, whereas LMNN seems better for
\alg{Ionosphere}, and both algorithms yield similar performance on \alg{Iris}.

The regularized versions of both algorithms (solid blue lines) significantly
improve the classification performance. Remarkably, regularized ITML shows almost no
degradation in classification performance, even in very high noise regimes, demonstrating a strong robustness to noise. 
%
%


 These results underscore the value of regularization in metric learning, showing that regularization encourages adaptability to the
intrinsic complexity and improved robustness to noise.

%
%

%
%

\section{Discussion and Conclusion}
\label{sec:related_work}

Previous theoretical work on metric learning has focused almost exclusively on
analyzing the generalization error of variants of the optimization criteria for
the distance-based metric learning framework. 

\citet{met_learn:analysis_regularized}, for instance, analyzed the generalization ability of
regularized, convex-loss optimization criteria for pairwise distances via an algorithmic \emph{stability} analysis. They derive an
interesting sample complexity result that is sublinear in $\sqrt{D}$ 
for datasets of representation dimension $D$. They discuss that the sample complexity can potentially be 
independent of $D$, but do not characterize specific instances or classes of problems where this may be possible.


Likewise, recent work by \citet{met_learn:analysis_robust_bellet} uses algorithmic
\emph{robustness} to analyze the generalization ability for pairwise- and triplet-based
distance metric learning. Their analysis relies on the existence of a partition of the input space, such that 
in each cell of the partition, the training loss and test loss does not deviate
much (robustness criteria). Note that their sample
complexity bound scales with the partition size, which in general can be
exponential in the representation dimension. 

Perhaps the works most similar to our approach are the sample complexity
analyses by \citet{met_learn:analysis_conv_loss} and
\citet{met_learn:analysis_cao_guo_ying}.
\citet{met_learn:analysis_conv_loss} analyze the consistency of the ERM
criterion for metric learning. They show a $O(m^{-1/2})$ rate of convergence for the ERM with $m$ samples to the expected risk 
for thresholds on bounded convex losses for distance-based metric
learning. Our upper-bound in Lemma \ref{lm:unif_conv_all} generalizes this
result by considering arbitrary (possibly non-convex) distance-based Lipschitz losses and explicitly shows the 
dependence on the representation dimension $D$.
\citet{met_learn:analysis_cao_guo_ying} provide an alternate analysis based on norm regularization of the weighting metric for distance-based metric learning.
Their result parallels our
norm-regularized criterion in Lemma \ref{lm:unif_conv_dist_refi}. While they focus on
analyzing a specific optimization criterion -- thresholds on the hinge loss with
norm-regularization, our result holds for general Lipschitz losses. 


It is worth emphasizing that none of these related works discuss the importance of or leverage the intrinsic structure in data for the metric learning problem. 
%
Our results in Section \ref{sec:discrim_complex} formalize an intuitive notion of dataset's intrinsic complexity for metric learning and show
sample complexity rates that are
finely tuned to this \emph{metric learning complexity}. 

The classifier-based framework we discuss has parallels with the kernel
learning literature. The typical focus in kernel learning is to analyze the generalization
ability of the hypothesis class of linear separators in general Hilbert spaces
\citep{kern_learn:analysis_ying_campbell, kern_learn:analysis_cortes_mohri}.
Our work provides a complementary analysis for learning explicit linear transformations of the given representation space for arbitrary hypotheses classes.

Our theoretical analysis partly justifies the empirical success of norm-based regularization as well. Our empirical results show that such regularization not only helps 
in designing new metric learning algorithms \citep{met_learn:alg_rob_stru_mcfee, met_learn:alg_fantope_regularization},
but can even benefit existing metric learning algorithms in high-noise regimes.


%


{
\bibliographystyle{icml2015}
\bibliography{mlearn}

\nocite{lt:anthony_bartlett}
\nocite{met_learn:analysis_robust_bellet}
\nocite{met_learn:analysis_conv_loss}
\nocite{met_learn:analysis_regularized}
\nocite{met_learn:analysis_cao_guo_ying}
\nocite{met_learn:analysis_guo_ying}
\nocite{met_learn:survey_bellet}
\nocite{met_learn:alg_ITML}
\nocite{met_learn:alg_LMNN}
\nocite{met_learn:alg_rob_stru_mcfee}
\nocite{kern_learn:analysis_cortes_mohri}
\nocite{kern_learn:analysis_ying_campbell}
\nocite{dataset:UCI_ML_repo}
}

\clearpage

\appendix
\section{Appendix: Various Proofs}
%
%

\subsection{Proof of Lemma \ref{lm:unif_conv_all}}

Let $\mathcal{P}$ be the probability measure induced by the random variable
$(\mathbf{X}, Y)$, where $\mathbf{X}:= (x,x')$, $Y:=\indicate[y=y']$, st. $((x,y),(x',y')) \sim (\D \times \D)$.

Define function class 
\begin{align*}
& \mathcal{F} := 
 \Bigg\{ f_M \!: \mathbf{X} \mapsto \|M(x-x')\|^2 \Bigg|\!  \begin{array}{c} M \in \mathcal{M} \\ \mathbf{X} = (x,x') \in (X \times X) \end{array} \!\! \Bigg\},
\end{align*}
and consider any loss function $\phi^\lambda(\rho, Y)$ that is $\lambda$-Lipschitz in the first argument.
Then, we are interested in bounding the quantity 
\begin{align*}
\sup_{f_M \in \mathcal{F}} \E_{(\mathbf{X},Y)\sim \mathcal{P}} [\phi^\lambda(f_M(\mathbf{X}), Y)] - \frac{1}{m} \sum_{i=1}^m \phi^\lambda(f_M(\mathbf{X}_i),Y_i),
\end{align*}
where $\mathbf{X}_i:=(x_{1,i},x_{2,i}) $, $Y_i:=\indicate[y_{1,i}=y_{2,i}]$
from the paired sample $S_m = \{((x_{1,i},y_{1,i}),(x_{2,i},y_{2,i}))
\}_{i=1}^m$.

Define $\bar{x}_i := x_{1,i} - x_{2,i}$ for each $ \mathbf{X}_i=(x_{1,i},x_{2,i})$. 
Then, the Rademacher complexity\footnote{See the definition of Rademacher complexity in the statement of Lemma \ref{lm:rad_complexities_unif_bound}.} of our function class $\mathcal{F}$ (with respect to the distribution $\mathcal{P}$) is bounded, since (let $\sigma_1,\ldots, \sigma_m$ denote
independent uniform $\{\pm 1\}$-valued random variables)
\begin{align*}
\mathcal{R}_m(\mathcal{F}, \mathcal{P}) 
& := \E_{\substack{\mathbf{X}_i,\sigma_i \\ i\in[m]}}  \Bigg[ \sup_{f_M \in \mathcal{F}} \frac{1}{m} \sum_{i=1}^m \sigma_i f_M(\mathbf{X}_i) \Bigg]  \\
&= \frac{1}{m} \E_{\substack{\mathbf{X}_i,\sigma_i \\ i\in[m]}} \sup_{M \in \mathcal{M}} \Big[ \sum_{i=1}^m \sigma_i \bar{x}_i^\mathsf{T} M^\mathsf{T}M \bar{x}_i   \Big] 
\\
&
= \frac{1}{m} \E_{\substack{\mathbf{X}_i,\sigma_i \\ i\in[m]}} \sup_{\substack{M \in \mathcal{M}, \textrm{ s.t.} \\ [a^{jk}]_{jk} := M^\mathsf{T}M }} \Bigg[ \sum_{j,k} a^{jk} \sum_{i=1}^m \sigma_i \bar{x}_i^j \bar{x}_i^k \Bigg] \\
&\leq \frac{1}{m} \E_{\substack{\mathbf{X}_i,\sigma_i \\ i\in[m]}} \sup_{M \in \mathcal{M}} \Bigg[ \|M^\mathsf{T}M\|_{_\textrm{F}} \Bigg( \sum_{j,k} \Big(\sum_{i=1}^m \sigma_i \bar{x}_i^j \bar{x}_i^k \Big)^2 \Bigg)^{1/2} \Bigg] \\
&\leq \frac{\sqrt{D}}{m} \E_{\substack{\mathbf{X}_i,  i\in[m]}} \Bigg( \E_{\substack{\sigma_i,i\in[m]}}  \sum_{j,k} \Big(\sum_{i=1}^m \sigma_i \bar{x}_i^j \bar{x}_i^k \Big)^2 \Bigg)^{1/2} \\
&= \frac{\sqrt{D}}{m} \E_{\substack{\mathbf{X}_i,  i\in[m]}} \Bigg( \sum_{j,k} \sum_{i=1}^m \Big(\bar{x}_i^j \Big)^2 \Big(\bar{x}_i^k \Big)^2 \Bigg)^{1/2} \\
&= \frac{\sqrt{D}}{m} \E_{\substack{\mathbf{X}_i,  i\in[m]}} \Bigg( \sum_{i=1}^m \|\bar{x}_i\|^4 \Bigg)^{1/2} \\
&= \frac{\sqrt{D}}{m} \E_{\substack{ (x_i, x'_i) \sim (\D|_X \times \D|_X) ,\\  i\in[m]}} \Bigg( \sum_{i=1}^m \|x_i - x'_i\|^4 \Bigg)^{1/2} \\
&\leq \sqrt{\frac{D}{m}} \Bigg (\E_{\substack{ (x, x') \sim (\D|_X \times \D|_X)}} \|x - x'\|^4 \Bigg)^{1/2} \\
&\leq 4B^2 \sqrt{\frac{D}{m}}, 
\end{align*}
where the second inequality is by noting that $\sup_{M\in\mathcal{M}} \|M^\mathsf{T}M\|_{_F} \leq \sqrt{D} $  for the class of weighting metrics $\mathcal{M} := \big\{M \;|\; M\in\R^{D\times D}, \sigma_{\max}(M)=1 \big\}$. 

Recall that $\D$ has bounded support (with bound $B$). Thus, by noting that $\phi^\lambda$ is $8B^2$ bounded function that is $\lambda$-Lipschitz in the first argument, we can apply Lemma \ref{lm:rad_complexities_unif_bound} and get the desired 
uniform deviation bound. \qed

\begin{lemma} \textbf{\emph{[Rademacher complexity of bounded Lipschitz loss functions \cite{lt:bartlett_mendelson_radgauss_complexities}] } }
\label{lm:rad_complexities_unif_bound}
Let $\D$ be a fixed unknown distribution over $X \times \{-1,1\}$, and
let $S_m$ be an i.i.d.\ sample of size $m$ from $\D$.
Given
a hypothesis class $\mathcal{H}\subset\R^{X}$ and a loss function 
$\ell : \R \times \{-1,1\} \rightarrow \R$, such that 
$\ell$ is $c$-bounded, and is $\lambda$-Lipschitz in the first argument, that
is, $\sup_{(y',y)\in\R\times\{-1,1\}} | \ell(y',y) | \leq c$, and $|\ell(y',y)
- \ell(y'',y)|\leq \lambda |y' - y''|$, we have the following:

for any $0<\delta<1$, with probability at least $1-\delta$, every $h\in \mathcal{H}$ satisfies
\begin{equation*}
\err(\ell \circ h,\D) \leq \err(\ell \circ h,S_m) + 2\lambda \mathcal{R}_m(\mathcal{H}, \D) + c\sqrt{\frac{2 \ln(1/\delta)}{m}},
\end{equation*}
where 
\begin{itemize}
\item $\err(\ell \circ h,\D) := \E_{{x,y}\sim \D} [\ell(h(x),y)]$,
\item $\err(h,S_m) := \frac{1}{m} \sum_{(x_i,y_i)\in S_m} \ell(h(x_i),y_i)   $,
\item 
$\mathcal{R}_m(\mathcal{H}, \D)$ is the Rademacher complexity of the function class $\mathcal{H}$ with respect to the distribution
$\D$ given $m$ i.i.d.\ samples, and is defined as:
\begin{equation*}
\mathcal{R}_m(\mathcal{H},\D) := \E_{\substack{ x_i\sim \D|_X, \\ \sigma_i \sim \mathrm{unif}\{\pm 1\}, \\ i\in[m] }} \Bigg[\sup_{{h}\in\mathcal{H}} \frac{1}{m} \sum_{i=1}^m \sigma_i h(x_i)   \Bigg],
\end{equation*}
where $\sigma_i$ are independent uniform $\{\pm 1\}$-valued random variables.
\end{itemize}
\end{lemma}

\subsection{Proof of Lemma \ref{lm:lb_dist}}

We shall exhibit a finite class of bounded support distributions $\classD$, such that if $\D$ is chosen uniformly at random from $\classD$, the expectation (over the random choice of $\D$) of the probability of failure (that is, generalization error of the metric returned by $\mathcal{A}$ compared to that of the optimal metric exceeds the specified tolerance level $\epsilon$) is at least $\delta$. This implies that for some distribution in $\classD$ the probability of failure is at least $\delta$ as well.

Let $\Delta_D := \{x_0,\ldots, x_{D}\}$ be a set of $D+1$ points that from the vertices of a regular unit-simplex from the underlying space $X=\R^D$ as per Definition
\ref{def:reg_simplex} (see below). For a fixed parameter $0<\alpha<1$ (exact value determined later), define $\classD$ as the class of all distributions $\D$
on $X \times \{0,1\}$ such that:
\begin{itemize}
\item $\D$ assigns zero probability to all sets not intersecting $\Delta_D \times \{0,1 \}$.
\item for each $i = 0,\ldots,D$, either
  \begin{itemize}
  \item $\Pr[(x_i,1)] = (1+\sqrt{\alpha})/2 $ and $\Pr[(x_i,0)] = (1-\sqrt{\alpha})/2$, or
  \item $\Pr[(x_i,1)] = (1-\sqrt{\alpha})/2 $ and $\Pr[(x_i,0)] = (1+\sqrt{\alpha})/2$.
  \end{itemize}
\end{itemize}

For concreteness, we shall use a specific instantiation of $\phi^\lambda_{L,U}$ in $\err_{\dist}^\lambda$ with $U=0$, $L=4/D$ and $\lambda=D/4$.
\\

\noindent \textbf{Proof overview.} We first show, by the construction of the
distributions under consideration in $\classD$, the sample error and the
generalization error minimizing metrics over any $\D \in \classD$ belong
to a restricted class of weighting matrices (Eq.\ \ref{eq:Mall_M01}). We then make a second simplification by noting that finding these (sample- and generalization-) error minimizing metrics (in the restricted class) is equivalent to solving a binary classification problem (Eq.\ \ref{eq:M01_binf}). This reduction to binary classification enables us to use VC-style lower bounding techniques to give a lower bound on the sample complexity. We now fill in the details.
\\

Consider a subset of weighting metrics $\Mgap$ that map points in $\Delta_D$ to exactly one of two possible points that are 
(squared) distance at least $4/D$ apart, that is,
\begin{align*}
\Mgap :=  \{ & M \;|\; M\in \mathcal{M}, \exists z_0, z_1 \in \R^D, \forall x\in\Delta_D, \\
& Mx \in \{z_0,z_1\} \textrm{ and }   \|z_0 - z_1\|^2\geq 4/D  \}.
\end{align*}

Now pick any $\D \in \classD$, let $S_m$ be an i.i.d.\ paired sample from $\D$. Observe that 
both the sample-based and the distribution-based error minimizing weighting metric from $\mathcal{M}$ on $\D$ also belongs to $\Mgap$. That is,
(c.f.\ Lemma \ref{lm:dist_lb_errall_eq_err01})
\begin{align}
\nonumber \argmin_{M\in\M}\err_{\dist}(M,\D) &= \argmin_{M\in\Mgap}\err_{\dist}(M,\D) \\
\argmin_{M\in\M}\err_{\dist}(M,S_m) &= \argmin_{M\in\Mgap}\err_{\dist}(M,S_m).
\label{eq:Mall_M01}
\end{align}

\noindent \textbf{A reduction to binary classification on product space.}
For each $M \in \Mgap$, we associate a classifier $f_M : (\Delta_D \times \Delta_D) \rightarrow \{0,1\}$ defined as
$(x_i,x_j) \mapsto \indicate[Mx_i = Mx_j]$. Now, consider the probability
measure $\mathcal{P}$ induced by the random variable $(\Xb, Y)$, where $\Xb := (x,x')$, $Y:=\indicate[y = y']$, s.t.\ $((x,y),(x',y')) \sim \big(\D|_{(\Delta_D \times \{0,1\})} \times \D|_{(\Delta_D \times \{0,1\})}\big)$. It is easy to check that for all $M \in \Mgap$
\begin{align}
\nonumber \err_{\dist}^\lambda(M,\D)  &=  \E_{(\Xb,Y) \sim \mathcal{P}}\big[ \indicate[f_M(\Xb)\neq Y] \big]\\
\err_{\dist}^\lambda(M,S_m) &=  \frac{1}{m} \!\! \sum_{((x,y),(x',y'))\in S_m} \!\!\!\!\!\!\!\!\!\!\!\!\! \indicate\big[f_M((x,x')) \neq \indicate[y=y']\big]. 
\label{eq:M01_binf}
\end{align}
Define 
\begin{eqnarray}
\nonumber
\eta(\Xb) &:=& \Pr_{Y\sim\mathcal{P}|_{Y|\Xb}}[Y=1|\Xb]  \\
\nonumber
&=& \Pr_{ (y,y') \sim (\D\times \D)|_{(y,y')|(x,x')}}[y=y'|x,x']  \\
&=&
\Bigg\{\begin{array}{ll} 
\vspace{1mm}
         \frac{1}{2}+\frac{\alpha}{2} & \textrm{if $\Pr(y|x)=\Pr(y'|x')$} \\
         \frac{1}{2}-\frac{\alpha}{2} & \textrm{if $\Pr(y|x)\neq \Pr(y'|x')$} 
      \end{array} . 
\label{eq:eta_val}
\end{eqnarray}
Observe that $\eta(\Xb)$ is the Bayes error rate at $\Xb$ for distribution
$\mathcal{P}$. Since, by construction of $\Mgap$, the class
$\{f_M\}_{M\in\Mgap}$ contains a classifier that achieves the Bayes error rate,
the optimal classifier $f^* := \argmin_{f_M}\E_{(\Xb,Y) \sim \mathcal{P}}\indicate[f_M(\Xb)\neq
Y]$ necessarily has $f^*(\Xb) = \indicate[\eta(\Xb)>\frac{1}{2}]$ (for all $\Xb$).
Then, for any $f_M$,
\begin{align}
\nonumber
\E_{(\Xb,Y) \sim \mathcal{P}}\big[ & \indicate[f_M(\Xb)\neq Y] \big]  -
\E_{(\Xb,Y) \sim \mathcal{P}}\big[ \indicate[f^*(\Xb)\neq Y] \big]
\\
\nonumber
&= \E_{\Xb \sim \mathcal{P}|_{\Xb}} \big[ \eta(\Xb) \big( \indicate[f^*(\Xb)=1] -\indicate[f_M(\Xb)=1] \big) 
\\
\nonumber
&\;\;\;\;\;\;\;\;\;\;\;\;\;\;\;\;\;
+ (1-\eta(\Xb)) \big( \indicate[f^*(\Xb)=0] -\indicate[f_M(\Xb)=0] \big) \big] \\
\nonumber
&= \E_{\Xb \sim \mathcal{P}|_{\Xb}} \big[ (2\eta(\Xb)-1) \big( \indicate[f^*(\Xb)=1] -\indicate[f_M(\Xb)=1] \big) \big] \\
\nonumber
&=  \E_{\Xb \sim \mathcal{P}|_{\Xb} }\big[ 2|\eta(\Xb)-{1}/{2}| \cdot \indicate[f_M(\Xb) \neq f^*(\Xb)]\big] \\
&= \frac{2\alpha}{(D+1)^2} \sum_{i>j} \big[\indicate[f_M((x_i,x_j)) \neq f^*((x_i,x_j))]\big], 
\label{eq:binf_valf}
\end{align}
where (i) the second to last equality is by noting that $f^*(\Xb) \neq 1 \iff \eta(\Xb)\leq 1/2$, and (ii) the last equality is by noting Eq.\ \eqref{eq:eta_val}, $f_M((x_i,x_i)) = f^*((x_i,x_i))=1$ for all $i$ and $f((x_i,x_j)) = f((x_j,x_i))$ for all $f$. 
For notational simplicity, we shall define $\Xb_{i,j} := (x_i,x_j)$. 

Now, for a given paired sample $S_m$, let $N(S_m) := (N_{i})_{i}$ (for all $0\leq i\leq D$),
where $N_{i}$ is the number of occurrences of the point
$x_i$ in $S_m$. 
Then for any $f_M$,
\begin{align*}
\E_{S_m}& \Bigg[  \frac{1}{(D+1)^2}\sum_{i>j}  \indicate[f_M(\Xb_{i,j}) \neq f^*(\Xb_{i,j})] \Bigg]
\\
&= 
\frac{1}{(D+1)^2}\sum_{i>j} \Pr_{S_m}[f_M(\Xb_{i,j}) \neq f^*(\Xb_{i,j})] \\
&= 
\frac{1}{(D+1)^2}\sum_{i>j} \sum_{N\in\mathbb{N}^{D+1}} \Pr_{S_m}[f_M(\Xb_{i,j}) \neq f^*(\Xb_{i,j}) | N(S_m) = N] \cdot 
\Pr[N(S_m) = N]  \\
&= 
\frac{1}{(D+1)^2}\sum_{N\in\mathbb{N}^{D+1}} \Pr[N(S_m) = N] \cdot 
\sum_{i>j} \Pr_{S_m}[f_M(\Xb_{i,j}) \neq f^*(\Xb_{i,j}) | N_{i}, N_{j}]  \\
& \geq
\frac{1}{(D+1)^2}\sum_{N\in\mathbb{N}^{D+1}} \Pr[N(S_m) = N] \cdot 
\sum_{i>j} \frac{1}{4} \Bigg(1-\sqrt{1-\exp{\Bigg( \frac{-(\max\{N_{i},N_{j}\} +1) \alpha^2}{1-\alpha^2}\Bigg)}} \Bigg) \\
&\geq 
\frac{1}{4}\frac{D}{D+1} \Bigg(1-\sqrt{1-\exp{\Bigg( \frac{-((2m/(D+1)) +1) \alpha^2}{1-\alpha^2}\Bigg)}} \Bigg)\\
&\geq 
\frac{1}{8} \Bigg(1-\sqrt{1-\exp{\Bigg( \frac{-((2m/(D+1)) +1) \alpha^2}{1-\alpha^2}\Bigg)}} \Bigg),
\end{align*}
where (i) the first inequality is by applying Lemma \ref{lm:two_coin_lb}, (ii) the second inequality is by assuming WLOG $N_i \geq N_j$, and  
noting that the expression above is convex in $N_{i}$ so one can apply Jensen's inequality and by observing that $\E[N_{i}]=2m/(D+1)$ and that there are total $D(D+1)$ summands for $i>j$, and (iii) the last inequality is by noting that $D\geq1$. Now, let $B$ denote the r.h.s.\ quantity above.
Then by recalling that for any $[0,1]$-valued random variable $Z$, $\Pr(Z>\gamma) > \E Z - \gamma$ (for all $0<\gamma<1$), we have
\begin{align*}
\Pr_{S_m}\Big[ \frac{1}{(D+1)^2}\sum_{i>j} \indicate[f_M((x_i,x_j))  \neq f^*&((x_i,x_j))]  > \gamma B \Big] 
> (1-\gamma) B.
\end{align*}
Or equivalently, by combining Eqs.\ \eqref{eq:Mall_M01}, \eqref{eq:M01_binf} and \eqref{eq:binf_valf}, we have
\begin{align*}
\E_{\D \sim \textrm{unif($\classD$)}} \Pr_{S_m \sim \D} & \Big[  \err_{\dist}(\mathcal{A}(S_m),\D) 
- \err_{\dist}(M^*_{\D},\D)   > 2\alpha \gamma B \Big] > (1-\gamma) B,
\end{align*}
where $M^*_{\D} := \argmin_{M\in\mathcal{M}} \err_{\dist}(M,\D)$ and $\mathcal{A}(S_m)$ is any metric returned by empirical error minimizing
algorithm. 
Now, if (cond.\ 1) $B\geq \delta/1-\gamma$ and (cond.\ 2) $\epsilon \leq 2
\gamma \alpha B$, it follows that for some $\D\in\classD$
\begin{eqnarray}
\Pr_{S_m \sim \D} \Big[ \err_{\dist}(\mathcal{A}(S_m),\D) - \err_{\dist}(M^*_{\D},\D) > \epsilon \Big] > \delta.
\label{eq:lb_main}
\end{eqnarray}
Now, to satisfy cond.\ 1 \& 2, we shall select $\gamma = 1-16\delta$. Then cond.\ 1 follows
if
$$
m\leq \frac{(D+1)}{2} \Bigg( \frac{1-\alpha^2}{\alpha^2} \ln(4/3) -1 \Bigg).
$$
Choosing parameter $\alpha = 8\epsilon/\gamma$ (and by noting $B\geq 1/16$ by cond.\ 1 for choice of $\gamma$ and $m$), cond.\ 2 is satisfied as well. Hence, 
$$
m\leq \frac{(D+1)}{2} \Bigg( \frac{(1-16\delta)^2-(8\epsilon)^2}{64\epsilon^2} \ln(4/3) -1 \Bigg)
$$
implies Eq.\ \eqref{eq:lb_main}. Moreover, if $0<\epsilon,\delta<1/64$ then $m\leq \frac{(D+1)}{512\epsilon^2}$ would suffice. \qed
\\

\begin{definition} 
\label{def:reg_simplex}
Define $n+1$ vectors $\Delta_n = \{v_0,\ldots,v_n \}$, with each $v_i \in \R^n$  as 
\begin{align*}
v_{0,j} &= \frac{-1}{\sqrt{n}}  &\textrm{for $1\leq j \leq n$} \\
v_{i,j} &= \Bigg\{ \begin{array}{ll} \frac{(n-1)\sqrt{n+1}+1}{n\sqrt{n}} & \textrm{if $i=j$} \\ \frac{-(\sqrt{n+1} -1)}{n\sqrt{n}} & \textrm{otherwise} \end{array} &\textrm{for $1\leq i,j \leq n$}
\end{align*}
\end{definition}

\begin{fact} 
\label{fact:simplex_properties}
\emph{\textbf{[properties of vertices of a regular $n$-simplex]}}
Let $\Delta_n = \{v_0,\ldots, v_n\}$ be a set of $n+1$ vectors in $\R^n$ as per Definition \ref{def:reg_simplex}. Then, $\Delta_n$ defines vertices of a regular $n$-simplex circumscribed in a unit $(n-1)$-sphere, with
\begin{enumerate}[(i)]
\item $\|v_i\|^2=1$ (for all $i$), and 
\item $\|v_i-v_j\|^2=2(n+1)/n$ (for $i\neq j$). 
\end{enumerate}
Moreover, for any non-empty bi-partition of $\Delta_n$ into $\Delta^{(1)}_n$ and
$\Delta^{(2)}_n$ with $|\Delta^{(1)}_n|=k$ and $|\Delta^{(2)}_n| = n+1-k$, define $a^{(1)}$ and $a^{(2)}$ the means (centroids) of the
points in $\Delta^{(1)}_n$ and $\Delta^{(2)}_n$ respectively. Then, we also have
\begin{enumerate}[(i)]
\item $(a^{(1)} - a^{(2)}) \cdot (a^{(i)} - v_j ) = 0$ (for $i\in\{1,2\}$, and $v_j \in \Delta^{(i)}_n$).
\item $\|a^{(1)} - a^{(2)} \|^2 = \frac{(n+1)^2}{kn(n+1-k)} \geq \frac{4}{n}$, for $1\leq k\leq n$.
\end{enumerate}
\end{fact}

\begin{lemma} 
\label{lm:dist_lb_errall_eq_err01}
Let $\Delta_D$ be a set of $D+1$ points $\{X_0,\ldots,X_D\}$ in $\R^D$ as per
Definition \ref{def:reg_simplex}, and let $\D$ be an arbitrary distribution
over $\Delta_D \times \{0,1\}$. Define $P_i := \indicate[\Pr_{\D}[(X_i,1)]>1/2]$.
Define $\Pi := \{\pi: \Delta_D \rightarrow
\R^D\}$ be the collection of all functions that maps points in $\Delta_D$ to
arbitrary points in $\R^D$. 
Define 
\begin{align*}
f((x,y),&(x',y');\pi) := 
\Bigg\{ 
  \begin{array}{ll} 
     \min\{1,\frac{D}{4} \|\pi(x) - \pi(x')\|^2\} & \textrm{if $y=y'$} \\
     \min\{1,[1-\frac{D}{4} \|\pi(x) - \pi(x')\|^2]_{_+}\} & \textrm{if $y\neq y'$} 
   \end{array}.
\end{align*}
Let $\mathcal{E}(\pi) := \E_{(x,y),(x',y')\sim \D\times \D}[f((x,y),(x',y');\pi)] $ and $\mathcal{E}^* := \inf_\pi \mathcal{E}(\pi)$. Then,
for any $\bar \pi \in \Pi$ such that
\begin{enumerate}[(i)]
\item $\bar \pi(X_i) = \bar \pi(X_j)$, if $P_i = P_j$
\item $\|\bar \pi(X_i) - \bar \pi(X_j)\|^2 \geq \frac{4}{D}$, if $P_i \neq P_j$,
\end{enumerate}
we have that $\mathcal{E}(\bar \pi) = \mathcal{E}^*$.  
Moreover, define $\bar A$ as
\begin{itemize}
\item $\bar A := \frac{A_1-A_0}{\|A_1-A_0 \|}$, where $A_0 := \mean(X_i)$ such that $P_i=0$, and $A_1 := \mean(X_i)$ such that $P_i = 1$ (if exists at least one $P_i=0$ and at least one $P_i = 1$). 
\item $\bar A := 0$, i.e.\ the zero vector in $\R^D$ (otherwise).
\end{itemize}
And let $M$ be a $D\times D$ matrix (with $\sigma_{\max}(M)=1$) defined as
$$
M := {\bar A} {\bar A}^{\T}.
$$
Then the map $\pi_M: x \mapsto M x$ constitutes a map that satisfies conditions (i) and (ii) and thus $\mathcal{E}(\pi_M) = \mathcal{E}^*$.
\end{lemma}
\begin{proof}
The proof follows from the geometric properties of $\Delta_D$ and Fact \ref{fact:simplex_properties}.
\end{proof}

\begin{lemma} 
\label{lm:two_coin_lb}
Given two random variables $\alpha_1$ and $\alpha_2$, each uniformly distributed on
$\{\alpha_-,\alpha_+\}$ independently, where $\alpha_- = 1/2-\epsilon/2$ and $\alpha_+ =
1/2+\epsilon/2$ with $0<\epsilon<1$. Suppose that $\xi_1^1,\ldots,\xi_m^1$ and $\xi_1^2,\ldots,\xi_m^2$ are
two i.i.d.\ sequences of $\{0,1\}$-valued random variables with $\Pr(\xi^1_i=1)=\alpha_1$ and $\Pr(\xi^2_i=1)=\alpha_2$ for all
$i$. Then, for any likelihood maximizing function $f$ from $\{0,1\}^m$ to $\{\alpha_-,\alpha_+\}$ that estimates the bias $\alpha_1$ and $\alpha_2$ from the samples, 
\begin{align*}
&\Pr\Big[ \big( f(\xi^1_1,\ldots,\xi^1_m)\neq \alpha_1 \textup{ and } f(\xi^2_1,\ldots,\xi^2_m) = \alpha_2 \big), 
\\
&\;\;\;\;\; 
\textup{or} 
\big( f(\xi^1_1,\ldots,\xi^1_m) = \alpha_1 \textup{ and } f(\xi^2_1,\ldots,\xi^2_m) \neq \alpha_2 \big)
 \Big] 
> \frac{1}{4}\Bigg(1-\sqrt{1-\exp{\Big( \frac{-2 \lceil m/2\rceil \epsilon^2}{1-\epsilon^2}\Big)}} \Bigg).
\end{align*}
\end{lemma}
\begin{proof} Note that
\begin{align*}
&\Pr\Big[ \big( f(\xi^1_1,\ldots,\xi^1_m)\neq \alpha_1 \textup{ and } f(\xi^2_1,\ldots,\xi^2_m) = \alpha_2 \big), 
\textup{or} 
\big( f(\xi^1_1,\ldots,\xi^1_m) = \alpha_1 \textup{ and } f(\xi^2_1,\ldots,\xi^2_m) \neq \alpha_2 \big)
\Big] 
\\
&\;\;\;
= \;\;
\Pr[f(\xi^1_1,\ldots,\xi^1_m)\neq \alpha_1]\cdot \Pr[f(\xi^2_1,\ldots,\xi^2_m) = \alpha_2 ]
+ \Pr[ f(\xi^1_1,\ldots,\xi^1_m) = \alpha_1] \cdot \Pr[f(\xi^2_1,\ldots,\xi^2_m) \neq \alpha_2 ] \\
&\;\;\;\geq \;
\frac{1}{2}\Pr\big[f(\xi^1_1,\ldots,\xi^1_m)\neq \alpha_1\big] + \frac{1}{2} \Pr\big[f(\xi^2_1,\ldots,\xi^2_m) \neq \alpha_2 \big] \\
&\;\;\;> \;\;
\frac{1}{4}\Bigg(1-\sqrt{1-\exp{\Big( \frac{-2 \lceil m/2\rceil \epsilon^2}{1-\epsilon^2}\Big)}} \Bigg),
\end{align*}
where the first inequality is by noting that a likelihood maximizing $f$ will select the correct bias better than random (which has probability $1/2$), and
the second inequality is by applying Lemma \ref{lm:coin_lb}.
\end{proof}

\begin{lemma} \emph{\textbf{[Lemma 5.1 of \citet{lt:anthony_bartlett}]}}
\label{lm:coin_lb}
Suppose that $\alpha$ is a random variable uniformly distributed on
$\{\alpha_-,\alpha_+\}$, where $\alpha_- = 1/2-\epsilon/2$ and $\alpha_+ =
1/2+\epsilon/2$, with $0<\epsilon<1$. Suppose that $\xi_1,\ldots,\xi_m$ are
i.i.d.\ $\{0,1\}$-valued random variables with $\Pr(\xi_i=1)=\alpha$ for all
$i$. Let $f$ be a function from $\{0,1\}^m$ to $\{\alpha_-,\alpha_+\}$. Then
$$
\Pr\big[ f(\xi_1,\ldots,\xi_m)\neq \alpha \big] \! > \! \frac{1}{4}\Bigg(\!1-\sqrt{1-\exp{\Big( \frac{-2 \lceil m/2\rceil \epsilon^2}{1-\epsilon^2}\Big)}} \Bigg).
$$
\end{lemma}


\subsection{Proof of Lemma \ref{lm:hypoth_ub}}
For any $M \in \mathcal{M}$ define real-valued hypothesis class on domain $X$ as $\mathcal{H}_M := \{x\mapsto h(Mx) : h\in \mathcal{H}\}$ and define
\begin{eqnarray*}
\mathcal{F} := \{ x \mapsto h(Mx) : M\in \mathcal{M}, h \in \mathcal{H} \} = \bigcup_M \mathcal{H}_M.
\end{eqnarray*}

Observe that a uniform convergence of errors induced by the functions in
$\mathcal{F}$ implies convergence of the class of weighted matrices as well.

Now for any domain $X$, real-valued hypothesis class $\mathcal{G} \subset [0,1]^X$, margin $\gamma>0$, and a sample $S\subset X$, define
\begin{equation*}
\cov_\gamma(\mathcal{G},S) := \Bigg\{C\subset \mathcal{G} \Big| \begin{array}{cc} \forall
g\in\mathcal{G}, \exists g'\in C,  \\ \max_{s\in S}|g(s)-g'(s)| \leq \gamma \end{array} \Bigg\}
\end{equation*}
as the set of $\gamma$-covers of $S$ by $\mathcal{G}$. Let $\gamma$-covering number of $\mathcal{G}$ for any integer $m>0$ be defined as
\begin{equation*}
\mathcal{N}_\infty(\gamma,\mathcal{G},m) := \max_{S\subset X: |S|=m} \min_{C\in \cov_\gamma(\mathcal{G},S)}|C|,
\end{equation*}
with the minimizing cover $C$ called as the minimizing $(\gamma,m)$-cover of $\mathcal{G}$
\\

Now, for the given $\gamma$, we will first estimate the $\gamma$-covering number of $\mathcal{F}$, that is, $\mathcal{N}_\infty(\gamma,\mathcal{F},m)$. 

For any $M\in \mathcal{M}$, let $H_M$ be the minimizing $(\gamma/2,m)$-cover of
$\mathcal{H}_M$. 
Note that $|H_M| = \mathcal{N}_\infty(\gamma/2,\mathcal{H}_M,m) \leq \mathcal{N}_\infty(\gamma/2,\mathcal{H},m)$ (because $MX \subset X$). 

Now
let $\mathcal{M}_\epsilon$ be an $\epsilon$-spectral cover of $\mathcal{M}$ (that is, for every $M\in \mathcal{M}$, exists $M' \in \mathcal{M}_\epsilon$ such that $\sigma_{\max}(M-M')\leq \epsilon$), and define 
\begin{eqnarray*}
\bar{F}_\epsilon := \{x\mapsto h(Mx): M \in \mathcal{M}_\epsilon, h\in \mathcal{H}_M  \}.
\end{eqnarray*}
Note that $|\bar F_\epsilon| \leq |\mathcal{M}_\epsilon| |H_I| \leq \mathcal{N}_\infty(\gamma/2,\mathcal{H},m) (1+2D/\epsilon)^{D^2}$ (c.f.\ Lemma \ref{lm:matrix_cover}). Observe that $\bar F_\epsilon$ is a 
$(\gamma/2 + B \lambda \epsilon)$-cover of $\mathcal{F}$, since (i) for any $f\in
F$ (formed by combining, say, $M_0 \in \mathcal{M}$ and $h_0\in \mathcal{H}$),
exists $\bar{f} \in \bar F_\epsilon$, namely the $\bar{f}$ formed by $\bar{M}_0$ such
that $\sigma_{\max}(M_0 - \bar{M}_0) \leq \epsilon$, and (ii) $\bar h_0 \in H_{\bar M_0}$ such that $|h_0(\bar M_0 x) - \bar
h_0(\bar M_0 x)|\leq \gamma/2$ (for all $x \in X$). So, (for any $x\in X$)
\begin{eqnarray*}
|f(x) - \bar f(x)| 
&=& |h_0(M_0 x) - \bar h_0(\bar M_0 x)|  \\
&\leq& |h_0(M_0 x) - h_0(\bar M_0 x)|\\&& + |h_0(\bar M_0 x)  - \bar h_0(\bar M_0 x)|  \\
&\leq& \lambda \|M_0 x - \bar M_0 x\| + \gamma/2  \\
&\leq& \lambda \sigma_{\max}(M_0 - \bar M_0) \|x\| + \gamma/2  \\
&\leq& \lambda \epsilon B + \gamma/2.  
\end{eqnarray*} 
So, if we pick $\epsilon = \min\{\frac{1}{2\lambda B}, \frac{\gamma}{2}\}$, it follows that
\begin{eqnarray*}
\mathcal{N}_\infty(\gamma,\mathcal{F},m) \leq |\bar F_\epsilon| \leq \mathcal{N}_\infty(\gamma/2,\mathcal{H},m) (1+2D/\epsilon)^{D^2}.
\end{eqnarray*}
By noting Lemmas \ref{lm:ab_cover_to_fat} and \ref{lm:ab_unif_conv_cover}, it follows that
\begin{align*}
\Pr_{S_m\sim\D} \Big[ \exists f \in \mathcal{F} : \err(f) & \geq \err_\gamma(f,S_m) + \alpha \Big]  
\\
&
\leq  4 \Big(1+\frac{2D}{\epsilon}\Big)^{\! D^2}\!\! \Big(\frac{128m}{\gamma^2}\Big)^{\fat_{\gamma/16}(\mathcal{H})\ln\big(\frac{32 e m}{\fat_{\gamma/16}(\mathcal{H})\gamma}\big)} e^{-\alpha^2m/8}.
\end{align*}
The lemma follows by bounding this failure probability with at most $\delta$.
\qed

\begin{lemma} \textbf{\emph{[$\epsilon$-spectral coverings of $D\times D$ matrices]}}
\label{lm:matrix_cover}
Let $\mathcal{M} := \{M \;|\; M\in \R^{D\times D}, \sigma_{\max}(M)=1\}$ be the set of matrices with unit spectral norm. 
Define $\mathcal{M}_\epsilon$ as the $\epsilon$-cover of $\mathcal{M}$, that is, for every $M \in \mathcal{M}$, there exists $M' \in \mathcal{M}_\epsilon$ such that $\sigma_{\max}(M-M') \leq \epsilon$.
Then for all $\epsilon>0$, there exists $\mathcal{M}_\epsilon$ such that
$|\mathcal{M}_\epsilon| \leq  \big(1 + \frac{2D}{\epsilon} \big)^{D^2}$.
\end{lemma}
\begin{proof}
Fix any $\epsilon > 0$ and let $\mathcal{N}_{\epsilon/D}$ be a minimal size $(\epsilon/D)$-cover of Euclidean unit ball $\mathbf{B}_D$ in $\R^D$. That is,
for any $v\in \mathbf{B}_D$, there exists $v' \in \mathcal{N}_{\epsilon/D}$ such that $\|v-v'\|\leq \epsilon/D$.
Using standard volume arguments (see e.g.\ proof of Lemma 5.2 of \citet{matrix:vershynin}), we know that
$|\mathcal{N}_{\epsilon/D}| \leq \big(1+\frac{2D}{\epsilon}\big)^D$. 
Define 
\begin{align*}
\mathcal{M}_\epsilon := \Big\{ M' \; \big| \; M' = [v'_1 \; \cdots \; v'_D] \in \R^{D\times D}, v'_i \in \mathcal{N}_{\epsilon/D} \Big\}.
\end{align*}
Then $\mathcal{M}_\epsilon$ constitutes as an $\epsilon$-cover of
$\mathcal{M}$, since for any $M = [v_1 \cdots v_D] \in \mathcal{M}$ there
exists $M' = [v'_1 \cdots v'_D] \in \mathcal{M}_\epsilon$, in particular $M'$ such that $\|v_i - v'_i\| \leq \epsilon/D$ (for all $i$). Then
\begin{align*}
\sigma_{\max}(M-M') \leq \|M-M'\|_{_F} = \sum_i \|v_i - v'_i\| \leq \epsilon.
\end{align*}
Without loss of generality we can assume that each $M'\in \mathcal{M}_\epsilon$, $\sigma_{\max}(M')=1$.
Moreover, by construction, $| \mathcal{M}_\epsilon| \leq \big(1+\frac{2D}{\epsilon}\big)^{D^2}$.
\end{proof}

\begin{lemma} \textbf{\emph{[extension of Theorem 12.8 of \citet{lt:anthony_bartlett}] }}
\label{lm:ab_cover_to_fat}
Let $\mathcal{H}$ be a set of real functions from a domain $X$ to the interval
$[0,1]$. Let $\gamma>0$. Then for all $m\geq 1$,
\begin{eqnarray*}
\mathcal{N}_\infty(\gamma,\mathcal{H},m) < c_0 (4m/\gamma^2)^{\fat_{\gamma/4}(\mathcal{H}) \ln \frac{4em}{\fat_{\gamma/4}(\mathcal{H})\gamma}}.
\end{eqnarray*}
for some universal constant $c_0$.
\end{lemma}
\begin{proof}
Theorem 12.8 of \citet{lt:anthony_bartlett} asserts this for $m\geq \fat_{\gamma/4}(\mathcal{H})\geq 1 $ with $c_0 = 2$. Now, if $1\leq m< \fat_{\gamma/4}(\mathcal{H})$, 
for some universal constant $c'$, we have $\mathcal{N}_\infty(\gamma,\mathcal{H},m) \leq (c'/\gamma)^m \leq (c'/\gamma)^{\fat_{\gamma/4}(\mathcal{H})}$. 
\end{proof}

\begin{lemma} \textbf{\emph{ [Theorem 10.1 of \citet{lt:anthony_bartlett}] }}
\label{lm:ab_unif_conv_cover}
Suppose that $\mathcal{H}$ is a set of real-valued functions defined on domain $X$. Let $\D$ be any probability distribution on
$Z = X \times \{0,1\}$, $0\leq \epsilon \leq 1$, real $\gamma>0$ and integer $m\geq 1$. Then,
\begin{align*}
\Pr_{S_m\sim\D}  \Big[ \exists h \in \mathcal{H} : \err(h) \geq & \err_\gamma(h,S_m) + \epsilon \Big]  
\leq 2 \mathcal{N}_\infty\Big(\frac{\gamma}{2},\mathcal{H},2m\Big) e^{-\epsilon^2m/8},
\end{align*}
where $S_m$ is an i.i.d.\ sample of size $m$ from $\D$.
\end{lemma}

\subsection{Proof of Lemma \ref{lm:hypoth_lb}}
For any fixed $0<\gamma<1/8$ and the given bounded class of distributions with bound $B\geq 1$, consider a $(1/B)$-bi-Lipschitz base hypothesis class $\mathcal{H}$ that maps hypothesis from the domain $X$ to $[1/2-4\gamma, 1/2+4\gamma]$, and define 
$$
  \mathcal{F} := \{x \mapsto h(Mx): M\in \mathcal{M}, h\in \mathcal{H}\}.
$$
Note that finding $M$ that minimizes $\err_{\hypoth}$ is equivalent to finding $f$ that minimizes error on $\mathcal{F}$.
Using Lemma \ref{lb:ab_fat_lb}, we have for any $0<\gamma<1/2$, the sample complexity of 
$\mathcal{F}$ is (for all $0<\epsilon,\delta<1/64$)
\begin{eqnarray}
m \geq \frac{\fat_{2\gamma}(\pi_{4\gamma}(\mathcal{F}))}{320\epsilon^2},
\label{eq:lb_hypoth_lbF}
\end{eqnarray}
where $\pi_{4\gamma}(\mathcal{F})$ is the $(4\gamma)$-\emph{squashed} function class of $\mathcal{F}$ (see Definition \ref{def:squash_fxn} below).
We lower bound $\fat_{2\gamma}(\pi_{4\gamma}(\mathcal{F}))$ in terms of fat-shattering dimension of $\mathcal{H}$ to yield the lemma. 

To this end
we shall first define the $(\gamma,m)$-covering and packing number of a generic real-valued hypothesis class $\mathcal{G}$.
For any domain $X$, real-valued hypothesis class $\mathcal{G} \subset [0,1]^X$, margin $\gamma>0$, and a sample $S\subset X$, define
\begin{align*}
\cov_\gamma(\mathcal{G},S) &:= \Bigg\{C\subset \mathcal{G} \Big| \begin{array}{cc} \forall
g\in\mathcal{G}, \exists g'\in C,  \\ \max_{s\in S}|g(s)-g'(s)| \leq \gamma \end{array} \Bigg\}, \\
\pak_\gamma(\mathcal{G},S) &:= \Bigg\{P\subset \mathcal{G} \Big| \begin{array}{cc} \forall
g \neq g' \in P,  \\ \max_{s\in S}|g(s)-g'(s)| \geq \gamma \end{array} \Bigg\} 
\end{align*}
as the set of $\gamma$-covers (resp.\ $\gamma$-packings) of $S$ by $\mathcal{G}$. Let $\gamma$-covering number (resp.\ $\gamma$-packing number) of $\mathcal{G}$ for any integer $m>0$ be defined as
\begin{align*}
\mathcal{N}_\infty(\gamma,\mathcal{G},m) & := \max_{S\subset X: |S|=m} \min_{C\in \cov_\gamma(\mathcal{G},S)}|C|, \\
\mathcal{P}_\infty(\gamma,\mathcal{G},m) & := \max_{S\subset X: |S|=m} \max_{P\in \pak_\gamma(\mathcal{G},S)}|P|
\end{align*}
with the minimizing cover $C$ (resp.\ maximizing packing $P$) called as the minimizing $(\gamma,m)$-cover (resp.\ maximizing $(\gamma,m)$-packing) of $\mathcal{G}$.
\\

With these definitions, we have the following (for some universal constant $c_0$).
\begin{align}
 \nonumber c_0  \Big(\frac{m}{16\gamma^2}\Big)^{\fat_{2\gamma}(\pi_{4\gamma}(\mathcal{F})) \ln(em/2\gamma)} 
& 
\nonumber \geq \mathcal{N}_\infty(8\gamma, \pi_{4\gamma}(\mathcal{F}),m) & \textrm{[Lemma \ref{lm:ab_cover_to_fat}]}\\
& \nonumber \geq \mathcal{P}_\infty(16\gamma, \pi_{4\gamma}(\mathcal{F}),m) & \textrm{[Lemma \ref{lm:ab_covering_packing}]}\\
& \nonumber \geq \Big(\frac{1}{32\gamma}\Big)^{D^2} \mathcal{P}_\infty(48\gamma, \pi_{4\gamma}(\mathcal{H}),m)  & \textrm{[see (*) below]}\\
& \nonumber = \Big(\frac{1}{32\gamma}\Big)^{D^2} \mathcal{P}_\infty(48\gamma, \mathcal{H},m)  & \textrm{[by the choice of $\mathcal{H}$]}\\
& \nonumber \geq \Big(\frac{1}{32\gamma}\Big)^{D^2} \mathcal{N}_\infty(48\gamma, \mathcal{H},m) & \textrm{[Lemma \ref{lm:ab_covering_packing}]}\\
& \geq \Big(\frac{1}{32\gamma}\Big)^{D^2} e^{\fat_{768\gamma}(\mathcal{H})/8}. & \textrm{[Lemma \ref{lm:ab_cover_to_fat_lb}]}
\label{eq:hypoth_fat_lb}
\end{align}

\noindent (*) We show that $\mathcal{P}_\infty(16\gamma, \pi_{4\gamma}(\mathcal{F}),m) \geq (1/32\gamma)^{D^2}
\mathcal{P}_\infty(48\gamma, \pi_{4\gamma}(\mathcal{H}),m) $, by exhibiting a
set $\mathcal{S}\subset\pi_{4\gamma}(\mathcal{F})$ of size $ (1/32\gamma)^{D^2}
\mathcal{P}_\infty(48\gamma, \pi_{4\gamma}(\mathcal{H}),m)$ that is a
$(16\gamma)$-packing of $\pi_{4\gamma}(\mathcal{F})$.

Let $\pi_{4\gamma}(\mathcal{H}_{48\gamma}) \subset \pi_{4\gamma}(\mathcal{H})$ be a maximal $(32\gamma)$-packing of
$\pi_{4\gamma}(\mathcal{H})$ (that is, a maximal set such that for all distinct $(\pi_{4\gamma}\circ h),(\pi_{4\gamma} \circ h')\in
\pi_{4\gamma}(\mathcal{H}_{48\gamma})$, exists $x\in X$ such that $|\pi_{4\gamma}(h(x)) - \pi_{4\gamma}(h'(x))| \geq
48\gamma$). Fix $\epsilon$ (exact value determined later), and define 
$$\mathcal{S}_\epsilon := \Bigg\{x\mapsto (\pi_{4\gamma}\circ h)(Mx) \; \Big| \;
\begin{array}{cc} (\pi_{4\gamma}\circ h)\in\pi_{4\gamma}(\mathcal{H}_{48\gamma}), \\ M\in \mathcal{M}_\epsilon \end{array} \Bigg\},$$
where $\mathcal{M}_\epsilon$ is a $\epsilon$-spectral net of $\mathcal{M}$, that is, for all $M \in \mathcal{M}$, exists $M' \in \mathcal{M}_\epsilon$ such
that $\sigma_{\max}(M-M') \leq \epsilon$, and for all distinct $M', M'' \in \mathcal{M}_\epsilon$, $\sigma_{\max}(M'-M'') \geq \epsilon/2$.

 Then for any two distinct
$f,f' \in \mathcal{S}_\epsilon$, such that $f(x) = (\pi_{4\gamma}\circ h)(Mx)$ and $f'(x) = (\pi_{4\gamma}\circ h')(M'x)$, we have
\begin{itemize}
\item (case 1) $h$ and $h'$ are distinct. In this case, there exists $x\in X$, s.t.\
\begin{align*}
|f(x) - f'(x)| =  & |\pi_{4\gamma}(h(Mx)) - \pi_{4\gamma}(h'(M'x))|\\
\geq & \; |\pi_{4\gamma}(h(Mx)) - \pi_{4\gamma}(h'(Mx))| \\ &-  |\pi_{4\gamma}(h'(Mx)) - \pi_{4\gamma}(h'(M'x))|\\
\geq &\; 48\gamma - (1/B) \sigma_{\max}(M-M') \|x\| \\
\geq &\; 48\gamma - (1/B) \epsilon B  = 48\gamma - \epsilon .
\end{align*}

\item (case 2) $h$, $h'$ same but $M$ and $M'$ distinct. In this case, there exists $x$ (with $\|x\|=1$) s.t.\
\begin{eqnarray*}
|f(x) - f'(x)| &=&  |\pi_{4\gamma}(h(Mx)) - \pi_{4\gamma}(h(M'x))|\\
&=&  |h(Mx) - h(M'x)| \\
&\geq& B \|(M - M')x\| \\
&\geq& B \cdot \min_{M\neq M' \in \mathcal{M}_\epsilon} \sigma_{\max} (M - M') \\
&\geq& B (\epsilon/2). 
\end{eqnarray*}
\end{itemize}
Thus, by setting $\epsilon = 32\gamma$, distinct classifiers $f,f' \in
\mathcal{S}_{32\gamma}$ are at least $16\gamma$ apart (since $B\geq 1$). Hence $\mathcal{S}_{32\gamma}$ forms a $(16\gamma)$-packing of
$\pi_{4\gamma}(\mathcal{F})$. Therefore, the packing number
\begin{align*}
\mathcal{P}_\infty(16\gamma,\pi_{4\gamma}(\mathcal{F}),m) &\geq
|\mathcal{S}_{32\gamma}| = |\mathcal{M}_{32\gamma}| |\mathcal{H}_{48\gamma}| 
\geq
(1/32\gamma)^{D^2}\mathcal{P}_\infty(48\gamma,\pi_{4\gamma}(\mathcal{H}),m). 
\end{align*}

\noindent Thus, from Eq.\ \eqref{eq:hypoth_fat_lb}, it follows that 
$$\fat_{2\gamma}(\pi_{4\gamma}(\mathcal{F})) \geq \Omega\Big( \frac{D^2 \ln(1/\gamma) + \fat_{768\gamma}(\mathcal{H})}{\ln(m/\gamma^2)\ln(m/\gamma)}\Big).$$
Combining this with Eq.\ \eqref{eq:lb_hypoth_lbF}, the lemma follows.
\qed

\begin{lemma} \textbf{\emph{[$\epsilon$-spectral packings of $D\times D$ matrices]}}
\label{lm:matrix_packing}
Let $\mathcal{M} := \{M \;|\; M\in \R^{D\times D}, \sigma_{\max}(M)=1\}$ be the set of matrices with unit spectral norm. 
Define $\mathcal{M}_\epsilon \subset \mathcal{M}$ as the $\epsilon$-packing of
$\mathcal{M}$, that is, for every distinct $M,M' \in \mathcal{M}_\epsilon$,
$\sigma_{\max}(M-M') \geq \epsilon$.
Then for all $\epsilon>0$, there exists $\mathcal{M}_\epsilon$ such that
$|\mathcal{M}_\epsilon| \geq  \big(\frac{1}{2\epsilon} \big)^{D^2}$.
\end{lemma}
\begin{proof}
Fix any $\epsilon > 0$ and let $\mathcal{P}_{\epsilon}$ be a maximal size $\epsilon$-packing of Euclidean unit ball $\mathbf{B}_D$ in $\R^D$. That is,
for all distinct $v, v'\in \mathbf{B}_D$, $\|v-v'\|\geq \epsilon$.
Using standard volume arguments (see e.g.\ proof of Lemma 5.2 of \citet{matrix:vershynin}), we know that
$|\mathcal{P}_{\epsilon}| \geq \big(\frac{1}{2\epsilon}\big)^D$. 
Define 
\begin{align*}
\mathcal{M}_\epsilon := \Big\{ M' \; \big| \; M' = [v'_1 \; \cdots \; v'_D] \in \R^{D\times D}, v'_i \in \mathcal{P}_{\epsilon} \Big\}.
\end{align*}
Then $\mathcal{M}_\epsilon$ constitutes as an $\epsilon$-packing of
$\mathcal{M}$, since for any distinct $M,M'\in \mathcal{M}_\epsilon$ such that $M = [v_1 \cdots v_D]$ and $ M' = [v'_1 \cdots v'_D]$, we have
\begin{align*}
\sigma_{\max}(M-M') \geq \max_{i} \|v_i -v'_i \| \geq \epsilon.
\end{align*}
Without loss of generality we can assume that each $M\in \mathcal{M}_\epsilon$, $\sigma_{\max}(M)=1$.
Moreover, by construction, $| \mathcal{M}_\epsilon| \geq \big(\frac{1}{2\epsilon}\big)^{D^2}$.
\end{proof}

\begin{lemma} \textbf{\emph{[follows from Theorem 12.1 of \citet{lt:anthony_bartlett}] }}
\label{lm:ab_covering_packing}
For any real valued hypothesis class $\mathcal{H}$ into $[0,1]$, all $m\geq 1$, and $0<\gamma<1/2$, 
$$
\mathcal{P}_\infty(2\gamma,\mathcal{H},m) \leq \mathcal{N}_\infty(\gamma,\mathcal{H},m)\leq \mathcal{P}_\infty(\gamma,\mathcal{H},m).
$$
\end{lemma}

\begin{lemma} \textbf{\emph{[Theorem 12.10 of \citet{lt:anthony_bartlett}]}}
\label{lm:ab_cover_to_fat_lb}
Let $\mathcal{H}$ be a set of real functions from a domain $X$ to the interval
$[0,1]$. Let $\gamma>0$. Then for $m\geq \fat_{16\gamma}(\mathcal{H})$,
\begin{eqnarray*}
\mathcal{N}_\infty(\gamma,\mathcal{H},m) \geq e^{\fat_{16\gamma}(\mathcal{H})/8}.
\end{eqnarray*}
\end{lemma}

\begin{lemma}\textbf{\emph{[Theorem 13.5 of \citet{lt:anthony_bartlett}]}}
\label{lb:ab_fat_lb}
Suppose that $\mathcal{H}$ is a set of real-valued functions mapping into the interval $[0,1]$ that is closed under addition of
constants, that is, 
$$h \in \mathcal{H} \implies h' \in \mathcal{H}, \textrm{ where }  h':x\rightarrow h(x)+c \;\;\; \textrm{ for all $c$.} $$
Pick any $0<\gamma<1/2$.
 Then for any metric learning algorithm $\mathcal{A}$ for
all $0<\epsilon,\delta<1/64$, there exists a distribution $\D$ such
that if
$m\leq\frac{d}{320\epsilon^2}$, then
$$
\Pr_{S_m\sim\D} [ \err(h^*,\D) > \err_\gamma(\mathcal{A}(S_m),\D) + \epsilon] >\delta
$$
where $d := \fat_{2\gamma}(\pi_{4\gamma}(\mathcal{H})) \geq 1 $ is the fat-shattering
dimension of $\pi_{4\gamma}(\mathcal{H})$---the $(4\gamma)$-\emph{squashed} function class of $\mathcal{H}$, see Definition \ref{def:squash_fxn} below---at margin $2\gamma$.
\end{lemma}

\begin{definition}  \textbf{\emph{[squashing function]}}
\label{def:squash_fxn}
For any $0<\gamma<1/2$, define the \emph{squashing function} $\pi_\gamma:\R \rightarrow [1/2-\gamma,1/2+\gamma]$ as
$$
\pi_\gamma(\alpha) = \Bigg\{ 
  \begin{array}{ll} 
     1/2+\gamma & \textrm{if $\alpha\geq 1/2+\gamma$} \\ 
     1/2-\gamma & \textrm{if $\alpha\leq 1/2-\gamma$} \\ 
     \alpha & \textrm{otherwise} \\ 
  \end{array} .
$$
Moreover, for a collection $F$ of functions into $\R$, define $\pi_\gamma(F) := \{ \pi_\gamma \circ f \; | \; f\in F\}$. 
\end{definition}

\subsection{Proof of Lemma \ref{lm:unif_conv_dist_refi}}

Let $\mathcal{P}$ be the probability measure induced by the random variable
$(\mathbf{X}, Y)$, where $\mathbf{X}:= (x,x')$, $Y:=\indicate[y=y']$, st. $((x,y),(x',y')) \sim (\D \times \D)$.

Define function class 
\begin{align*}
& \mathcal{F} := 
\Bigg\{ f_M \!: \mathbf{X} \mapsto \|M(x-x')\|^2 \Bigg|\!  \begin{array}{c} M \in \mathcal{M} \\ \mathbf{X} = (x,x') \in (X \times X) \end{array} \!\! \Bigg\},
\end{align*}

Following the steps of proof of Lemma \ref{lm:unif_conv_all}, we can conclude that the Rademacher complexity of $\mathcal{F}$ is bounded. In particular,
\begin{align*}
\mathcal{R}_m(\mathcal{F}) \leq 4B^2 \sqrt{ \frac{ \sup_{M \in \mathcal{M}} \|M^\mathsf{T}M\|^2_{_F}  } {m}}.
\end{align*}
The result follows by noting that $\phi$ is $\lambda$-Lipschitz in the first argument and by applying Lemma \ref{lm:rad_complexities_unif_bound}.
\qed

\subsection{Proof of Lemma \ref{lm:unif_conv_clf_refi}}

Consider the function class 
\begin{align*}
\mathcal{F} := \Big\{ f_{v,M} : x \mapsto v \cdot Mx \; \big| \; \|v\|_1 \leq 1, M \in \mathcal{M} \Big\},
\end{align*}
and define the composition class
\begin{align*}
\mathcal{F}_\sigma := \Bigg\{ x \mapsto \sum_{i=1}^K w_i \sigma^\gamma(f_i(x)) \; \Big| \;\begin{array}{c} \|w_i\|_1 \leq 1, \\ f_1,\ldots, f_K \in \mathcal{F} \end{array} \Bigg\}.
\end{align*}

Then, first note that the Gaussian complexity of $\mathcal{F}$ (with respect to the distribution $\mathcal{D}$) is bounded, since (let $g_1,\ldots,g_m$ denote independent standard Gaussian random variables)
\begin{align*}
\mathcal{G}_m(\mathcal{F}, \mathcal{D}) & :=  \E_{\substack{ x_i \sim \D|_X \\g_i, i\in[m]}} \Bigg[ \sup_{f_{v,M}\in \mathcal{F}} \frac{1}{m} \sum_{i=1}^m g_i f_{v,M}(x_i) \Bigg] \\
&=\frac{1}{m} \E_{\substack{x_i \sim \D|_X \\ g_i, i\in[m]}}  \Bigg[ \sup_{\substack{M\in\mathcal{M} \\ \|v\|_1\leq 1}} v \cdot \sum_{i=1}^m g_i (Mx_i) \Bigg] \\
&=\frac{1}{m} \E_{\substack{x_i \sim \D|_X \\ g_i, i\in[m]}} \Bigg[ \max_j \sup_{M\in\mathcal{M}} \sum_{i=1}^m g_i (Mx_i)_j \Bigg] \\
&\leq \frac{1}{m} \E_{\substack{x_i \sim \D|_X \\ g_i, i\in[m]}} \max_{j\in[D]} \Bigg[ \sum_{i=1}^m g_i \sup_{M\in\mathcal{M}} \big|(Mx_i)_j\big| \Bigg] \\
&\leq \frac{c \ln^{\frac{1}{2}}(D)}{m} \E_{x_i \sim \D_X} \max_{j,j'\in[D]} \Bigg( \E_{g_i } \Bigg[ \sum_{i=1}^m g_i \Big( \sup_{M\in\mathcal{M}} \big|(Mx_i)_j\big| 
- \sup_{M'\in\mathcal{M}} \big|(M'x_i)_{j'}\big| \Big) \Bigg]^2 \Bigg)^{\frac{1}{2}} 
\\
&= \frac{c \ln^{\frac{1}{2}}(D)}{m} {\E_{x_i \sim \D_X}} \max_{j,j'\in[D]} \Bigg( \sum_{i=1}^m \Big[  \sup_{M\in\mathcal{M}} \big|(Mx_i)_j\big| 
- \sup_{M'\in\mathcal{M}} \big|(M'x_i)_{j'}\big| \Big]^2 \Bigg)^{\frac{1}{2}} \\
&\leq  c'B \sqrt{\frac{d\ln{D}}{m}},
\end{align*}
where (i) second to last inequality is by applying Lemma \ref{lm:bartlett_mendelson_slepian_lemma}, (ii) $c,c'$ are absolute constants, (iii) $d := \sup_{M\in \mathcal{M}} \| M^\mathsf{T}M \|^2_{_F}$. Note that bounding the Gaussian complexity also bounds the Rademacher
complexity by Lemma \ref{lm:bartlett_mendelson_relate_rag_gauss}.

Finally by noting that $\mathcal{F}_\sigma$ is a $\gamma$-Lipschitz composition class of $\mathcal{F}$ and $\phi^\lambda$ is a classification based loss function that is $\lambda$-Lipschitz in the first argument, we can apply Lemma \ref{lm:rad_complexities_unif_bound} yielding the desired result.
\qed

\begin{lemma} \emph{\textbf{[Lemma 20 of \citet{lt:bartlett_mendelson_radgauss_complexities}]}}
\label{lm:bartlett_mendelson_slepian_lemma}
Let $Z_1,\ldots,Z_D$ be random variables such that each $Z_j = \sum_{i=1}^m a_{ij} g_i$, where each $g_i$ is independent $N(0,1)$
random variables. Then there is an absolute constant $c$ such that 
\begin{align*}
\E_{g_i} \max_j Z_j \leq c\ln^{\frac{1}{2}}(D) \max_{j,j'} \sqrt{\E_{g_i}(Z_j - Z_{j'})^2}.
\end{align*}
\end{lemma}

\begin{lemma} \emph{\textbf{[Lemma 4 of \citet{lt:bartlett_mendelson_radgauss_complexities}]}}
\label{lm:bartlett_mendelson_relate_rag_gauss}
There are absolute constants $c$ and $C$ such that for every class $\mathcal{F}$ and every integer $m$
\begin{align*}
c \mathcal{R}_m(\mathcal{F},\mathcal{D}) \;\leq\; \mathcal{G}_m(\mathcal{F},\mathcal{D}) \;\leq\; C \ln(m)  \mathcal{R}_m(\mathcal{F},\mathcal{D}),
\end{align*}
where $\mathcal{R}$ and $\mathcal{G}$ are Rademacher and Gaussian complexities of a function class $\mathcal{F}$ with respect to the distribution $\mathcal{D}$ respectively. 
\end{lemma}

\subsection{Proof of Corollary \ref{cor:unif_conv_refined}}
The conclusion of Eq.\ \eqref{eq:main_unif_refind_bound} is immediate by
dividing the given failure probability $\delta$ across the sequence
$\mathcal{M}^1,\mathcal{M}^2,\cdots$ such that $\delta \mu_d$ failure
probability is associated with class $\mathcal{M}^d$, then apply Lemma
\ref{lm:unif_conv_dist_refi} (for distance based metric learning) or Lemma
\ref{lm:unif_conv_clf_refi} (for classifier based metric learning) to each
class $\mathcal{M}^d$ individually, and finally combining the individual deviations together with a union
bound. 
\\

%


For the second part, for any $M\in \mathcal{M}$ define $d_M$ and $\Lambda_M$ as per the lemma statement. Then with probability at least $1-\delta$
\begin{align*}
\err^\lambda( M_m^{\reg},\D)  - \err^\lambda(M^*,\D) 
& \leq \;\; \err^\lambda(M_m^{\reg},S_m) + d_{_{M_m^{\reg}}}\Lambda_{_{M_m^{\reg}}} - \err^\lambda(M^*,\D) \\
& \leq \;\; \err^\lambda(M^*,S_m) + d_{_{M^*}}\Lambda_{_{M^*}} - \err^\lambda(M^*,\D) \\
& \leq \;\; O(d_{_{M^*}}\Lambda_{_{M^*}}) \;\; = \;\; O(\epsilon), 
\end{align*}
where (i) the first inequality is by applying Eq.\
\eqref{eq:main_unif_refind_bound} on weighting metric $M_m^{\reg}$ (with
failure probability set to $\delta/2$), (ii) the second inequality is by noting
that $M_m^{\reg}$ is the (regularized) sample error minimizer as per the lemma
statement, (iii) the third inequality is by applying Eq.\
\eqref{eq:main_unif_refind_bound} on weighting metric $M^*$ (with failure
probability set to $\delta/2$), and (iv) the last equality by noting the
definitions of $\Lambda_{M^*}$ and our choice of $m$.
\qed


\end{document}